\documentclass[oneside]{article}
\usepackage{geometry}
\usepackage{authblk}
\usepackage{lineno,hyperref}
\usepackage{natbib}
\usepackage{times}
\usepackage{lineno,hyperref}
\usepackage{amsmath, amsthm}
\usepackage{amsfonts, amssymb, bbm}
\usepackage{algorithm, algorithmic}
\usepackage[mathscr]{euscript}
\usepackage{mathtools, nicefrac}
\usepackage{booktabs}
\usepackage{multirow}
\usepackage[font=footnotesize]{subcaption}
\usepackage{pgfplots}
\usepgfplotslibrary{fillbetween}
\usepgfplotslibrary{external}
\pgfplotsset{compat=1.17}
\newcommand{\R}{\mathbb{R}}
\newcommand{\E}{\mathbb{E}}
\newcommand{\p}{\mathbb{P}}
\newtheorem{definition}{Definition}
\newtheorem{lemma}{Lemma}
\newtheorem{assumption}{}

\newtheorem{theorem}{Theorem}
\newtheorem{corollary}{Corollary}
\newtheorem{remark}{Remark}
\DeclareMathOperator{\argmin}{argmin}

\makeatletter
\renewcommand\section{\@startsection {section}{1}{\z@}%
    {18\p@ \@plus 6\p@ \@minus 3\p@}%
    {9\p@ \@plus 6\p@ \@minus 3\p@}%
    {\normalsize\bfseries\boldmath}}
\renewcommand\subsection{\@startsection{subsection}{2}{\z@}%
    {12\p@ \@plus 6\p@ \@minus 3\p@}%
    {3\p@ \@plus 6\p@ \@minus 3\p@}%
    {\normalfont\normalsize\itshape}}
\renewcommand\subsubsection{\@startsection{subsubsection}{3}{\z@}%
    {12\p@ \@plus 6\p@ \@minus 3\p@}%
    {\p@}%
    {\normalfont\normalsize\itshape}}

\global\bibsep=0pt

\allowdisplaybreaks
\title{Smoothed functional-based gradient algorithms \\for off-policy reinforcement learning:\\A non-asymptotic viewpoint}

\author{Nithia Vijayan\thanks{nithiav@cse.iitm.ac.in} }
\author{Prashanth L. A.\thanks{prashla@cse.iitm.ac.in}}
\affil{\normalsize \textit{Department of Computer Science and Engineering, Indian Institute of Technology Madras, India}}
\date{}
\begin{document}
    \maketitle
\begin{abstract}
    We propose two policy gradient algorithms for solving the problem of control in an off-policy reinforcement learning (RL) context. Both algorithms incorporate a smoothed functional (SF) based gradient estimation scheme. The first algorithm is a straightforward combination of importance sampling-based off-policy evaluation with SF-based gradient estimation. The second algorithm, inspired by the stochastic variance-reduced gradient (SVRG) algorithm, incorporates variance reduction in the update iteration. For both algorithms, we derive non-asymptotic bounds that establish convergence to an approximate stationary point. From these results, we infer that the first algorithm converges at a rate that is comparable to the well-known REINFORCE algorithm in an off-policy RL context, while the second algorithm exhibits an improved rate of convergence.
\end{abstract}
\section{Introduction}
\label{sec:intro}
In a reinforcement learning (RL) problem, an agent learns to achieve a goal through interactions with an environment. The interactions between the agent and the environment are represented as a Markov decision process (MDP). The agent interacts with the environment through actions, and as a response the environment changes its state and provides a reward. The goal of the agent is to maximize the cumulative reward over time by learning an optimal policy to choose actions.

We consider the problem of control in an off-policy RL setting, where the agent aims to learn an optimal policy using the data collected by executing an exploratory policy called behavior policy.
Off-policy RL is useful in practical scenarios where the system may not allow execution of any policy other than a fixed behavior policy. While the behavior policy may not be optimal, it can be exploratory, and aids in the search for the optimal policy.

Policy gradient algorithms \cite{williams,sutton1999policy,marbach2001simulation,kakade2001natural,shalabh2004spsaRL,papini2018,xu2020,zhangK2020,agarwal2020} are a popular approach for solving MDPs. In a few special cases such as linear systems with quadratic cost, policy gradient algorithms can be shown to be globally convergent \cite{bhandari2019,fazel2018,mohammadi2021}. In the general case, the usual convergence guarantees for a policy gradient algorithm are to a stationary point of the underlying value function (cf. \cite{papini2018,zhangK2020}). In \cite{mei2020,agarwal2020,bhandari2020}, the authors analyze policy gradient methods in the idealized setting where the gradient information is made directly available, while we consider a typical off-policy RL setting where the gradient of the objective has to estimated from a sample path of the behavior policy. Most of the previous works use the likelihood ratio method,  proposed in \cite{rubinstein1969some}, see \citep{Fu2006b,Fu2015} for an introduction. This approach for estimating the policy gradient was first used in a policy optimization context in the REINFORCE algorithm \cite{williams}. REINFORCE style gradient estimate methods are analyzed in \cite{zhangJ2020, liu2020}. While \cite{zhangJ2020} uses log barrier regularization, \cite{liu2020} analyzes a natural and variance-reduced counterparts of the policy gradient algorithm. The likelihood ratio method leads to unbiased estimates of the policy gradient.

An alternative approach for gradient estimation is the simultaneous perturbation method, see \cite{shalabh_book} for a textbook introduction. This method is based on finite differences, and results in a biased estimate of the policy gradient. A popular algorithm in this class is simultaneous perturbation stochastic approximation (SPSA), proposed in \cite{spall1992}. Using the classic finite difference type estimate of the policy gradient, i.e., a scheme that perturbs each co-ordinate separately, would require $2d$ function measurements, where $d$ is the dimension of the policy parameter. On the other hand, the SPSA scheme used random perturbations, e.g., a vector of independent Rademacher random variables (r.v.s), to simultaneously perturb all co-ordinates, and this scheme would work with two function measurements, irrespective of the dimension.
SPSA has been used in a policy gradient algorithm in \cite{shalabh2004spsaRL,shalabh2010}. Smoothed functional (SF) \cite{katkovnik1972,nesterov2017} is another simultaneous perturbation method, where one could employ a vector of independent standard Gaussian r.v.s as random perturbations.

In this paper, we propose two policy gradient algorithms for off-policy control. For the purpose of policy evaluation, both algorithms use the importance sampling ratios --- a standard scheme for unbiased off-policy evaluation. Unlike previous works on off-policy RL, our algorithms incorporate a SF-based gradient estimate scheme. We use the two function measurements variant of SF, which is equivalent to
evaluating two perturbed policies. In an on-policy RL setting, SF-based approach may be restrictive owing to the fact that
running two system trajectories corresponding to two perturbed policies may not feasible in some practical applications. On the other hand, using a SF-based policy gradient scheme does not run into practical difficulties in an off-policy RL context, since the system is simulated using a single behavior policy.

The first algorithm, henceforth referred to as OffP-SF, is a straightforward combination of importance sampling-based off-policy evaluation with SF-based gradient estimation. The second algorithm is inspired by the SVRG algorithm, which was proposed in \cite{johnson13} for optimizing finite `strongly convex' sum of smooth functions, and later adapted to a non-convex optimization setting (cf. \cite{reddi16,allen2016variance}). This algorithm, referred to as OffP-SF-SVRG, is the variance-reduced variant of the OffP-SF algorithm.
To the best of our knowledge, a variance-reduced policy gradient algorithm inspired by SVRG has not been proposed/analyzed in an off-policy RL context in the literature, while SVRG has been explored in the context of on-policy RL in \cite{papini2018,xu2020}. Recent work in \cite{mgh} explores variance reduction in an off-policy context inspired by a momentum-based method \cite{storm}.

In this paper, we focus on the non-asymptotic performance of the proposed algorithms. The results for policy gradient methods employing simultaneous perturbation-based gradient estimates are asymptotic in nature (cf. \cite{shalabh2004spsaRL,shalabh2010}). On the other hand, using ideas from zeroth-order optimization, policy gradient methods with REINFORCE style gradient estimates have been shown to converge to an $\epsilon$-stationary point (see Definition \ref{def:esolution} below) in the non-asymptotic regime. In this paper, we study policy gradient algorithms with the simultaneous perturbation approach, and derive non-asymptotic bounds for these algorithms --- see Table \ref{tab:summary} for a summary of our bounds in terms of iteration complexity, which is the number of policy gradient iterations required to find an $\epsilon$-stationary point. The primary conclusions from our non-asymptotic analysis are as follows: (i) After $N$ iterations of OffP-SF, the value function gradient at a suitably chosen iterate, say $\theta_R$, satisfies an order $O(\frac{1}{\sqrt{N}})$ bound on $\mathbb{E} ||\nabla J(\theta_R)||^2$; (ii) The corresponding bound for OffP-SF-SVRG is of the order $O(\frac{1}{N})$.

\begin{table}[h]
	\caption{Iteration complexity for our proposed algorithms, and the off-policy variant of REINFORCE. Here iteration complexity denotes the number of iterations required to find an $\epsilon$-stationary point (see Definition \ref{def:esolution}).}
	\label{tab:summary}
	\centering
	\begin{tabular}{c| c}
		\toprule
		Algorithm & Iteration complexity  \\
		\midrule
		OffP-REINFORCE\footnotemark &  $O(1/\epsilon^2)$\\\midrule
		OffP-SF  & $O(1/\epsilon^2)$ \\\midrule
		OffP-SF-SVRG & $O(1/\epsilon)$ \\
		\bottomrule
	\end{tabular}
\end{table}
 \footnotetext{This variant uses importance sampling ratios for off-policy evaluation, and the likelihood ratio method for gradient estimation.}

 Our bounds have a few advantages over those in the literature for zeroth-order optimization and on-policy RL using policy gradient algorithms. To elaborate, the closest result to the non-asymptotic bound for offP-SF is Corollary 3.3 of \cite{ghadimi2013}. For setting the step-size/perturbation constant in this result, one requires knowledge of quantities that are typically unknown in an RL setting. On the other hand, our non-asymptotic bound features a universal step-size/perturbation constant. In arriving at this result, we depart from the argument employed in the proof of Corollary 3.3 of \cite{ghadimi2013}.
 Our bound for OffP-SF in Corollary \ref{cr:non_asym} is comparable to the one provided in Corollary 4.4 in \cite{zhangK2020}, as both results are on the size of the gradient of the objective $J$ at a suitably chosen policy iterate. We employ smoothed functional based gradient estimation, while the authors in \cite{zhangK2020} use a REINFORCE style gradient estimate. Their result is for a diminishing step-size, while we employ a constant step-size.
Next,  the gradient estimates underlying the SVRG-based on-policy RL algorithms in \cite{papini2018,xu2020} use the likelihood ratio method, which result in unbiased estimates. On the other hand, our OffP-SF-SVRG algorithm employs smoothed functional-based gradient estimates, which are biased in nature. Through a careful handling of the bias terms in several steps of the proof, we are able to obtain an order $O(\frac{1}{N})$ bound for the OffP-SF-SVRG algorithm. The corresponding results for on-policy SVRG algorithm in \cite{papini2018,xu2020} features additional terms --- see the discussion below Theorem \ref{tm:non_asym_svrg} for more details.

In \cite{degris2012, zhangS2019, zhangS2020}, the authors propose actor-critic algorithms in an off-policy RL setting. In comparison, we do not incorporate function approximation in our proposed algorithms, and hence, a direct comparison is not feasible. Nevertheless, we mention that the algorithms in these references involve at least two timescales, and to the best of our knowledge, there are no non-asymptotic bounds for two timescale stochastic approximation, with a non-linear update iteration (as in the case of the actor update in the aforementioned references). In contrast, our algorithms operate on a single timescale, facilitating a non-asymptotic analysis. In \citep{liu2020}, the authors establish global convergence results for natural and variance-reduced counterparts of the policy gradient algorithm, with REINFORCE style gradient estimates. These results are under an assumption that the underlying policy parameterization is sufficiently rich. In contrast, we study local convergence properties of the vanilla and variance-reduced variants of the policy gradient algorithm, with smoothed functional-based gradient estimates. Finally, our non-asymptotic bound for OffP-SF-SVRG shows improved dependence on the number of iterations, as compared to the bound in \cite{mgh}, where the authors analyze a momentum-based variance reduced policy gradient scheme in an off-policy context.


This paper is an extended version of an earlier work (see \cite{nv2021}). Although the order of the convergence bounds remains the same, this version corrects errors from the earlier work by revising the proofs.

The rest of the paper is organized as follows: Section \ref{sec:pblm} describes the off-policy control problem. Section \ref{sec:offspsa} introduces our algorithms. Section \ref{sec:main} presents the non-asymptotic bounds for our algorithms. Section \ref{sec:conv} provides detailed proofs of convergence. Finally, Section \ref{sec:conclusions} provides the concluding remarks.

\section{Problem formulation}
\label{sec:pblm}
We consider an MDP with a state space $\mathscr{S}$, and an action space $\mathscr{A}$, both assumed to be finite. We operate in an episodic setting with a random episode length $T\in\mathbb{N}$. At time $t \in \{0, \hdots, T-1\}$, the MDP is in state $S_t$, and transitions to state $S_{t+1}$ by an action $A_t$ chosen by a behavior policy $b$, and receives a reward $R_{t+1}\in\mathbb{R}$. We assume the rewards are bounded, the start state $S_0$ is fixed. We also assume a special state $0$ as a termination state.

Let $\Theta$ be a compact and convex subset of $\mathbb{R}^d$. We consider parameterized stochastic target policies $\{\pi_\theta, \;\theta\in\Theta\}$, where $\pi_{\theta}(a\!\mid\! s)=\mathbb{P}\{A_t=a \!\mid\! S_t=s,\theta_t=\theta\}$. As an example, one may use an exponential softmax distribution, i.e.,
\begin{align*}
	\pi_{\theta}(a\!\mid\! s)= \frac{exp\left(h(s,a,\theta)\right)}{\sum_{b\in \mathscr{A}}exp\left(h(s,b,\theta)\right)},
\end{align*}
where $h:\mathscr{S}\times\mathscr{A}\times\Theta \rightarrow \mathbb{R}$ is a parameterized user defined function (cf. Chapter 13 of \cite{sutton_book}).
We assume that each policy in the parameterized class $\Theta$, and the behavior policy are proper (see \ref{as:proper}).

We assume that the MDP trajectory terminates under $\pi_\theta$ w.p. $1$, $\forall \theta \in \Theta$. The goal here is to find $\theta^*$ such that
\begin{align}
	\label{eq:max_theta}
	\theta^*\in\textrm{argmax}_{\theta \in \Theta} J(\theta),
\end{align}
where $J(\theta)$ is the value function, and is defined as
\begin{align}
	\label{eq:j_theta}
	J(\theta) = \mathbb{E}_{\pi_\theta}\left[\sum_{t=0}^{T-1}\gamma^t R_{t+1}\right],
\end{align}
where $\gamma \in (0,1]$ is the discount factor.

\section{Off-policy gradient algorithms}
\label{sec:offspsa}
A gradient-based algorithm for solving \eqref{eq:max_theta} would involve the following update iteration:
\begin{align}
	\label{eq:theta_update_1}
	\theta_{k+1} = \Pi_{\Theta}(\theta_k + \alpha_k \nabla J(\theta_k)),
\end{align}
where $\theta_0$ is set arbitrarily, and $\nabla$ is with respect to $\theta$.
In the above, the step-size $\alpha_k \in (0,1]$, and $\Pi_\Theta:\R^d \to \Theta$ is an operator that projects on to $\Theta$. The projection is required to ensure stability of the iterates in \eqref{eq:theta_update_1}, and is common in the analysis of policy gradient algorithms (cf. \cite{shalabh2004spsaRL}).
As an example, one may define $\Theta=\prod_{i=1}^d[\theta_{\min}^i, \theta_{\max}^i]$. Then, the  projection operator $\Pi_\Theta(\theta) = [\Pi_{\Theta}^1(\theta^1), \cdots, \Pi_{\Theta}^d(\theta^d)]$, where $\Pi_{\Theta}^i(\theta^i) = \min(\max(\theta_{\min}^i,\theta^i),\theta_{\max}^i), i\in \{1 \hdots d\}$.  It is easy to see that such a projection operation is computationally inexpensive.

We describe two algorithms for solving \eqref{eq:max_theta} below.
\subsection{OffP-SF}
In an off-policy setting, the distribution of data (states/actions seen along a sample path) follows the behavior policy. The off-policy evaluation problem is to learn the value of a target policy, which is different from the behavior one.
A standard off-policy evaluation scheme is per-decision importance sampling (see Section 5.9 of \cite{sutton_book}).  Here, one scales the objective by the likelihood ratio of the target policy, say $\pi_\theta$ to the behavior policy, say $b$ at the current state.
More precisely,  we generate $m$ episodes using $b$  and estimate  $J(\theta)$ as follows:
\begin{align}
    \label{eq:Jm}
    \hat{J}_{m}(\theta) \!=\! \frac{1}{m}\!\sum_{j=1}^{m} \!\sum_{t=0}^{T^j-1}\gamma^t R_{t+1}^j \!\left(\prod_{i=0}^{t}\frac{\pi_{\theta}(A_i^j\!\mid\!S_i^j)}{b(A_i^j\!\mid\!S_i^j)}\! \right).
\end{align}
In the above, $T^j$ is the length of the $j^{th}$ episode, and $R_{t+1}^j$ is the reward at time $t\!+\!1$. Also $S_i^j$ is the state, and $A_i^j$ is the action taken at time $i$ of the $j^{th}$ episode.

For estimating the gradient $\nabla J(\cdot)$, we employ the estimation scheme from \cite{katkovnik1972,nesterov2017}.
The idea here is to form a smoothed functional, denoted by $J_\mu$, of the value $J(\cdot)$, and use $\nabla J_\mu$ as a proxy for $\nabla J$. To be more precise, the smoothed functional  $J_\mu(\theta)$ is defined by
\begin{align}
	\label{eq:J_mu}
	J_\mu(\theta) = \mathbb{E}_{u \in \mathbb{B}^d}\left[J(\theta+\mu u)\right],
\end{align}
where $\mu \in (0,1]$ is the smoothing parameter, and $u$ is sampled uniformly at random from a unit ball $\mathbb{B}^d=\{x\in\mathbb{R}^d \mid \lVert x \rVert \leq 1\}$. Here $\lVert \cdot \rVert$ denotes the $d$-dimensional Euclidean norm.

We estimate the gradient using two randomly perturbed policies (cf. \cite{liu,shamir}). We favor the `balanced' estimate based on two random perturbations instead of a one-sided estimate because the bound on the second moment of the balanced estimate exhibits a linear dependence on the underlying dimension $d$ (see Lemma \ref{lm:var}), while the corresponding dependence in an one-sided estimate is quadratic in $d$ (see Proposition 7.6 of \cite{gao2018}).

We perturb the policy parameter $\theta$ by adding and subtracting a scalar multiple of a random unit vector $v$. The perturbed policy parameters lie in the set $\Theta'$ defined as follows:
\begin{align}
 \Theta'= \{\theta':\lVert\theta' - \theta \rVert \leq 1, \theta \in \Theta\}.\label{eq:thetaprime}
 \end{align}
In order to control the variance, we average the gradient estimate over $n$ random unit vectors.
The estimate $\widehat{\nabla}_{n,\mu} \hat{J}_{m}(\theta)$ of the gradient $\nabla J(.)$ is formed as follows:
\begin{align}
	\label{eq:nabla_Jm}
	\widehat{\nabla}_{n,\mu} \hat{J}_{m}(\theta)\!=\!\frac{d}{n}\sum_{i=1}^{n}\frac{\hat{J}_{m}(\theta+\mu v_i) \!-\! \hat{J}_{m}(\theta - \mu v_i)}{2\mu}v_i,
\end{align}
where $\forall i, v_i $ is sampled uniformly at random from a unit sphere $\mathbb{S}^{d-1}=\{x\in\R^d \mid \lVert x \rVert = 1\}$.

We collect $m$ sample paths using the behavior policy $b$, and use this data to estimate the value associated with the $2n$ perturbed policies in \eqref{eq:nabla_Jm}.


We solve \eqref{eq:max_theta} using the following update iteration:
\begin{align}
	\label{eq:approx_theta_update}
	\theta_{k+1} = \Pi_\Theta(\theta_k + \alpha_k \widehat{\nabla}_{n_k,\mu_k} \hat{J}_{m}(\theta_k)).
\end{align}

Algorithm \ref{alg:tab} presents the pseudocode of OffP-SF algorithm, with the following ingredients: (i) a gradient ascent update according to \eqref{eq:approx_theta_update}; (ii) a SF-based gradient estimation scheme; and (iii)  an importance sampling-based policy evaluation scheme.
\begin{algorithm}
    \caption{OffP-SF}
    \label{alg:tab}
    \begin{algorithmic}[1]
        \STATE \textbf{Input}: Parameterized form of target policy $\pi$ and behavior policy $b$, iteration limit $N$, step-sizes $\{\alpha_k\}$, perturbation constants $\{\mu_k\}$, batch size $m$, $\{n_k\}$, and probability mass function (pmf) $P_R(\cdot)$ supported on $\{1,\cdots,N\}$;
        \STATE \textbf{Initialize}: Target policy parameter $\theta_{0} \in \R^d$, and the discount factor $\gamma \in (0,1)$;
        \FOR {$k=0,\hdots, N-1$ }
        \FOR {$j=1, \hdots, m $}
        \STATE Get $(S_0^j, A_0^j, R_1^j, \!\cdots\!, S_{T_j-1}^j, A_{T_j-1}^j, R_{T_j}^j) \sim b$;
        \ENDFOR
        \FOR {$i=1,\hdots, n_k$ }
        \STATE Get $[v_i^1, \hdots, v_i^d] \in \mathbb{S}^{d-1}$;
        \STATE Use \eqref{eq:Jm} to estimate $\hat{J}_{m}(\theta_k \pm \mu_k v_i)$;
        \ENDFOR
        \STATE Use \eqref{eq:nabla_Jm} to estimate $\widehat{\nabla}_{n_k,\mu_k} \hat{J}_{m}(\theta_k)$;
        \STATE Use \eqref{eq:approx_theta_update} to calculate $\theta_{k+1}$;
        \ENDFOR
        \STATE \textbf{Output}: Policy $\theta_R$ where $R \sim P_R$.
    \end{algorithmic}
\end{algorithm}
\subsection{OffP-SF-SVRG}
Our second algorithm is a modification of the Algorithm \ref{alg:tab} that incorporates the concept of variance reduction seen in SVRG algorithms \cite{johnson13,reddi16}.
The principle of variance reduction underlying the SVRG algorithm has been explored in the context of on-policy RL in \cite{papini2018,xu2020}. The gradient estimates underlying the algorithms in the aforementioned references use the likelihood ratio method, which results in unbiased estimates. On the other hand, we employ SF-based gradient estimates, which are biased in nature.

We use nested update iterations to solve \eqref{eq:max_theta}. Our algorithm features an outer loop iterating over $s\in\{0,\cdots,S-1\}$, and an inner loop iterating over $k\in\{0,\cdots,l-1\}$. The policy parameters are of the form $\theta_k^s$, where $\theta_0^0$ is set arbitrarily.

In the outer loop, we sample $m$ episodes using the behavior policy $b$. We use a reference point $\tilde{\theta}^s\in\Theta$, which is initialized to $\theta_0^0$, and is updated as $\tilde{\theta}^{s+1}= \theta^s_m$. We calculate $\hat{J}_{m}(\tilde{\theta}^s)$ and $\widehat{\nabla}_{n,\mu} \hat{J}_{m}(\tilde{\theta}^s)$ using \eqref{eq:Jm} and \eqref{eq:nabla_Jm} respectively. Also $\forall j \in \{1,\cdots,m\}$, we calculate $\hat{J}^j(\tilde{\theta}^s)$ and $\widehat{\nabla}_{n,\mu} \hat{J}^j(\tilde{\theta}^s)$, where
\begin{align}
	\label{eq:Jhat_j}
	\hat{J}^j(\theta) = \sum_{t=0}^{T^j-1}\gamma^t R_{t+1}^j \left(\prod_{i=0}^{t}\frac{\pi_{\theta}(A_i^j|S_i^j)}{b(A_i^j|S_i^j)} \right),
\end{align}
and
\begin{align}
	\label{eq:nabla_Jhat_j}
	\widehat{\nabla}_{n,\mu} \hat{J}^j(\theta)=\frac{d}{n}\sum_{i=1}^{n}\frac{\hat{J}^j(\theta+\mu v_i) - \hat{J}^j(\theta - \mu v_i)}{2\mu}v_i.
\end{align}
In the above, $\forall i, v_i $ is sampled uniformly at random from a unit sphere $\mathbb{S}^{d-1}$.

In the inner loop, we pick a sample $j$ uniformly at random from $\{1,\cdots,m\}$ and calculate $\hat{J}^j(\theta_k^s)$ and $\widehat{\nabla}_{n,\mu} \hat{J}^j(\theta_k^s)$ using \eqref{eq:Jhat_j} and \eqref{eq:nabla_Jhat_j} respectively.

We update the policy parameters as follows:
\begin{align}
	\label{eq:approx_theta_update_svrg}
	\theta_{k+1}^s = \Pi_\Theta(\theta_k^s + \alpha g_k^s),
\end{align}
where
\begin{align}
    \label{eq:g}
    g_k^s=\widehat{\nabla}_{n,\mu} \hat{J}^j(\theta_k^s) - \widehat{\nabla}_{n,\mu} \hat{J}^j(\tilde{\theta}^s) + \widehat{\nabla}_{n,\mu}\hat{J}_{m}(\tilde{\theta}^s).
\end{align}

Algorithm \ref{alg:svrg} presents the pseudocode of OffP-SF-SVRG algorithm.
\begin{algorithm}
	\caption{OffP-SF-SVRG}
	\label{alg:svrg}
	\begin{algorithmic}[1]
		\STATE \textbf{Input}: Parameterized form of target policy $\pi$ and behavior policy $b$, iteration limit $S$, step-size $\alpha$, perturbation constant $\mu$, batch sizes $m, n$, and a joint pmf $P_{QR}(\cdot,\cdot)$ supported on $\!\{0,\cdots,S-1\}\!$ and $\!\{0,\cdots, l-1\}\!$ respectively;
		\STATE \textbf{Initialize}: Target policy parameter $\tilde{\theta}^{0}=\theta_0^0\in \R^d$, and the discount factor $\gamma \in (0,1)$;
		\FOR {$s=0,\hdots, S-1$ }
		\FOR {$j=1, \hdots, m $}
		\STATE Get $(S_0^j, A_0^j, R_1^j,\!\cdots\!, S_{T_j-1}^j, A_{T_j-1}^j, R_{T_j}^j) \sim b $;
		\ENDFOR
		\FOR {$i=1,\hdots, n$ }
		\STATE Get $[v_i^1, \hdots, v_i^d] \in \mathbb{S}^{d-1}$;
		\FOR {$j=1, \hdots, m $}
		\STATE Use \eqref{eq:Jhat_j} to estimate $\hat{J}^j(\tilde{\theta}^s \pm \mu v_i)$;
		\ENDFOR
		\STATE Use \eqref{eq:Jm} to estimate $\hat{J}_{m}(\tilde{\theta}^s \pm \mu v_i)$;
		\ENDFOR
		\FOR {$j=1,\hdots, m$ }
		\STATE Use \eqref{eq:nabla_Jhat_j} to estimate $\widehat{\nabla}_{n,\mu} \hat{J}^j(\tilde{\theta}^s)$;
		\ENDFOR
		\STATE Use \eqref{eq:nabla_Jm} to estimate $\widehat{\nabla}_{n,\mu} \hat{J}_{m}(\tilde{\theta}^s)$;
		\FOR {$k=0,\hdots, l-1$ }
		\STATE Get $j \in [1,m]$ uniformly and at random.
		\FOR {$i=1,\hdots, n$ }
		\STATE Use \eqref{eq:Jhat_j} to estimate $\hat{J}^j(\theta_k^s \pm \mu v_i)$;
		\ENDFOR
		\STATE Use \eqref{eq:nabla_Jhat_j} to estimate $\widehat{\nabla}_{n,\mu} \hat{J}^j(\theta_k^s)$;
		\STATE Use \eqref{eq:g} to calculate $g_k^s$;
		\STATE Use \eqref{eq:approx_theta_update_svrg} to calculate $\theta_{k+1}^s$;
		\ENDFOR
		\STATE $\tilde{\theta}^{s+1}=\theta_0^{s+1}=\theta_{l}^s$;
		\ENDFOR
		\STATE \textbf{Output}: Policy $\theta_{R}^{Q}$ where $Q,R \sim P_{QR}$.
	\end{algorithmic}
\end{algorithm}
\subsection{OffP-REINFORCE}
In REINFORCE \citep{Williams1992}, which is a well-known policy gradient algorithm, the gradient estimation scheme is based on the LR method. In principle, one could employ IS-based policy evaluation together with a REINFORCE style gradient estimate. We compare our algorithms to an off-policy variant of the REINFORCE algorithm, henceforth referred to as OffP-REINFORCE.
The OffP-REINFORCE algorithm employs the following update iteration:
\begin{align}
    \label{eq:theta_update_pg}
    \theta_{k+1} = \Pi_\Theta\left(\theta_k + \alpha \widehat{\nabla} J(\theta)\right),
\end{align}
where $\widehat{\nabla} J(\theta)$ is defined as follows:
\begin{align}
    \label{eq:hat_nabla_j_theta}
    \widehat{\nabla} J_m(\theta) &= \frac{1}{m}\sum_{j=1}^{m}\left[\sum_{t=0}^{T^j-1}\nabla\log\pi_\theta(A_t^j\mid S_t^j)
    \left(\prod_{i=0}^{t}\frac{\pi_\theta(A_i^j\mid S_i^j)}{b(A_i^j\mid S_i^j)}\right)\left(\sum_{i=t}^{T^j-1}\gamma^i R_{i+1}^j\right)\right].
\end{align}

Algorithm \ref{alg:reinforce} presents the pseudocode of OffP-REINFORCE algorithm.
\begin{algorithm}
    \caption{OffP-REINFORCE}
    \label{alg:reinforce}
    \begin{algorithmic}[1]
        \STATE \textbf{Input}: Parameterized form of target policy $\pi$ and behavior policy $b$, iteration limit $N$, step-sizes $\{\alpha_k\}$, perturbation constants $\{\mu_k\}$, batch size $m$, $\{n_k\}$, and probability mass function $P_R(\cdot)$ supported on $\{0,\cdots,N-1\}$;
        \STATE \textbf{Initialize}: Target policy parameter $\theta_{0} \in \Theta$, and the discount factor $\gamma \in (0,1)$;
        \FOR {$k=0,\hdots, N-1$ }
        \FOR {$j=1, \hdots, m $}
        \STATE Get $(S_0^j, A_0^j, R_1^j, \!\cdots\!, S_{T_j-1}^j, A_{T_j-1}^j, R_{T_j}^j) \sim b$;
        \ENDFOR
        \STATE Use \eqref{eq:hat_nabla_j_theta} to estimate $\widehat{\nabla} J_m(\theta_k)$;
        \STATE Use \eqref{eq:theta_update_pg} to calculate $\theta_{k+1}$;
        \ENDFOR
        \STATE \textbf{Output}: Policy $\theta_R$ where $R \sim P_R$.
    \end{algorithmic}
\end{algorithm}
\section{Main results}
\label{sec:main}
We make the following assumptions for the sake of analysis:
    \begin{assumption}
        \label{as:pol_cont}
        For any $a \in \mathscr{A}$ and $s \in \mathscr{S}$, $\log\pi_\theta(a \!\mid\! s)$ exists, and is twice continuously differentiable w.r.t. $\theta \in \Theta'$, where $\Theta'$ is defined in \eqref{eq:thetaprime}.
    \end{assumption}
    \begin{assumption}
        \label{as:b_pol}
        For every $\theta \in \Theta'$, the target policy $\pi_\theta$ is absolutely continuous with respect to the behavior policy $b$. i.e.,
        \begin{align*}
            \forall \theta \in \Theta',\; b(a \!\mid\! s)\!=\!0 \Rightarrow \pi_\theta(a \!\mid\! s)\!=\!0,\;\forall a \in \mathscr{A}, \forall s \in \mathscr{S}.
        \end{align*}
    \end{assumption}
\begin{assumption}
    \label{as:proper}
    The behavior policy $b$, and the class of target policies $\{\pi_\theta, \theta \in \Theta'\}$ are proper,
    i.e.,  there exists a positive constant $M$ s.t.
    \begin{align*}
        &\forall \theta \in \Theta',\; \max_{s\in\mathscr{S}}\;\mathbb{P}\left(S_M \neq 0 \mid S_0=s,\pi_\theta \right)<1, \textrm{ and}\\
        &\max_{s\in\mathscr{S}}\;\mathbb{P}\left(S_M \neq 0 \mid S_0=s, b \right)<1.
    \end{align*}
\end{assumption}
An assumption like \ref{as:pol_cont} is common to the analysis of policy gradient algorithms (cf. \cite{papini2018,xu2020}), while \ref{as:b_pol} is a standard requirements for off-policy evaluation. Further, \ref{as:proper} is a common requirement in the analysis of episodic MDPs, see Chapter 2 of \cite{ndp_book}. From \ref{as:pol_cont} and \ref{as:b_pol}, we have $\pi_{\theta}(a|s)>0$ and $b(a|s) >0$, $\forall \theta \in \R^d, \forall a \in \mathscr{A}, \textrm{ and } \forall s \in \mathscr{S}$.
In other words, we consider policies that place a positive mass on every action in any state.
The objective $J$ is not necessarily convex in a  typical RL setting, and hence, several previous works (cf. \cite{zhangK2020,papini2018,xu2020,shen2019hessian}) adopt convergence to an approximate stationary point,  which is defined below.
\begin{definition} \textbf{($\epsilon$-stationary point)}
	\label{def:esolution}
	Let $ \theta_R$ be the output of an algorithm. Then, $ \theta_R $ is called an $ \epsilon $-stationary point of problem \eqref{eq:max_theta}, if \ $ \mathbb{E} \left\| \nabla J \left( \theta _ { R } \right) \right\| ^ { 2 } \le \epsilon $.
\end{definition}
The non-asymptotic bounds for Algorithms \ref{alg:tab}--\ref{alg:svrg} that we present below establish convergence to an $\epsilon$-stationary point.

For the non-asymptotic analysis, we rewrite the update rule in \eqref{eq:approx_theta_update} as follows:
\begin{align}
	&\theta_{k+1} = \theta_k + \alpha_k \mathcal{P}_\Theta(\theta_k,\widehat{\nabla}_{n_k,\mu_k} \hat{J}_{m_k}(\theta_k), \alpha_k),\label{eq:gd_nonasym}
	\shortintertext{where}
	&\mathcal{P}_\Theta(\theta, f(\theta), \alpha) = \frac{1}{\alpha}\left[\Pi_{\Theta}(\theta+\alpha f(\theta))-\theta\right].\label{eq:proxy}
\end{align}
\begin{theorem}[\textbf{\textit{OffP-SF}}]
	\label{tm:non_asym}
	Assume \ref{as:pol_cont}--\ref{as:proper}. Let $P_R(k)=\mathbb{P}(R=k)=\frac{\alpha_k }{\sum_{k=0}^{N-1}\alpha_k}$, $\forall N \in \mathbb{N}$, and $J^*=\max_{\theta\in\Theta} J(\theta)$. Then,
	\begin{align}
		&\mathbb{E}\left[\left\lVert \mathcal{P}_\Theta(\theta_R, \nabla J (\theta_R), \alpha_R) \right\rVert^2\right]\nonumber\\
        &\leq \frac{\left(J^* - J(\theta_{0}) \right)+\frac{dL^2}{2}\sum_{k=0}^{N-1} \alpha_k \mu_k
            + \sqrt{2\pi}e^2dL^2 \sum_{k=0}^{N-1} \frac{\alpha_k}{\sqrt{n_k}} +\frac{d^2L^3}{2}\sum_{k=0}^{N-1}\alpha_k^2}{\sum\limits_{k=0}^{N-1}\alpha_k},\label{eq:nonasym-gen-bd}
	\end{align}
	where $L$ is the Lipschitz constant of $J$ as well as $\nabla J$ (see Lemma \ref{lm:J_lip} in Section \ref{sec:conv} below).
\end{theorem}
\begin{proof}
	See Section \ref{sec:conv}.
\end{proof}
The result above holds for any choice of step-sizes $\{\alpha_k\}$, perturbation constants $\{\mu_k\}$, and batch sizes $m$, $\{n_k\}$. We specialize the bound in \eqref{eq:nonasym-gen-bd} for a particular choice of the aforementioned parameters in the corollary below.
\begin{corollary}[\textbf{\textit{OffP-SF}}]
	\label{cr:non_asym}
    Set $\forall k$, $\alpha_k=\frac{1}{\sqrt{N}}$, $\mu_k=\frac{1}{\sqrt{N}}$, $n_k=N$, and $0<m < \infty$. Then, under conditions of Theorem \ref{tm:non_asym}, we have
        \begin{align*}
            &\mathbb{E}\left[\left\lVert \mathcal{P}_\Theta(\theta_R, \nabla J (\theta_R), \alpha_R) \right\rVert^2\right]
            \leq\frac{J^* - J(\theta_{0})}{\sqrt{N}}+\frac{dL^2+d^2L^3}{2\sqrt{N}} + \frac{2\sqrt{2\pi}e^2dL^2 }{\sqrt{N}}.
        \end{align*}
\end{corollary}
\begin{proof}
	See Section \ref{sec:conv}.
\end{proof}
\begin{remark}
Ignoring the error due to projection, i.e., assuming $\mathcal{P}_\Theta(\theta_R, \nabla J (\theta_R), \alpha_R) = \nabla J (\theta_R)$, the bound above can be read as follows: after $N$ iterations of \eqref{eq:theta_update_1}, OffP-SF returns an iterate that satisfies $\mathbb E\left\lVert\nabla J (\theta_R)\right\rVert^2 = O\left(\frac{1}{\sqrt{N}}\right)$. The closest result in a zeroth-order smooth non-convex optimization context is Corollary 3.3 of \cite{ghadimi2013}. In comparison to this result, our bound has a few advantages. First, the step-size in Corollary \ref{cr:non_asym} is set using a universal constant, while they require the knowledge of the smoothness parameter $L$. Second, the perturbation constant in Corollary \ref{cr:non_asym} is set using a universal constant, while the corresponding choice in \cite{ghadimi2013} requires the knowledge of $J^*-J(\theta_0)$. In a typical RL setting, one could possibly approximate $L$, but $J^*-J(\theta_0)$ is usually unknown.
\end{remark}
\begin{remark}
        While our theoretical guarantees are optimal for $m=1$, users have the flexibility to choose any value within the range $0 < m < \infty$. In practical implementations, averaging multiple observations, even though each is an unbiased estimate, makes the estimate more precise and reliable.
\end{remark}
%


Now, we present a non-asymptotic bound for Algorithm \ref{alg:svrg}.
For the  analysis, we rewrite the update rule in \eqref{eq:approx_theta_update_svrg} as follows:
\begin{align}
	\theta_{k+1}^s = \theta_k^s + \alpha \mathcal{P}_\Theta(\theta_k, g_k^s, \alpha),\label{eq:gd_nonasym_svrg}
\end{align}
where $\mathcal{P}_\Theta(\cdot, \cdot, \cdot)$ is as defined in \eqref{eq:proxy}.

\begin{theorem}[\textbf{\textit{OffP-SF-SVRG}}]
	\label{tm:non_asym_svrg}
	Assume \ref{as:pol_cont}--\ref{as:proper}. Let $P_{QR}(s,k)=\mathbb{P}(Q=s, R=k)=\frac{1 }{Sl}$, and $J^*=\max_{\theta\in\Theta} J(\theta)$.
        Set $\alpha=\frac{1}{6 dL}$, $\mu=\frac{1}{\sqrt{Sl}}$, $n=Sl$, $l=d$, and $0<m < \infty$. Then,
        \begin{align*}
            &\mathbb{E}\left[\left\lVert \mathcal{P}_\Theta(\theta_R^Q, \nabla J (\theta_R^Q), \alpha) \right\rVert^2\right]
            \leq
            \frac{2\left(J^*-J(\theta^{0}_{0})\right)}{Sl\tilde{c}}
            +\frac{384e^2d^2L^2}{Sl}+\frac{2 d^2L^2}{Sl}
            +\frac{2L}{S^2l^2\tilde{c}} +\frac{224e^2dL}{3 Sl\tilde{c}}, \\
            &\qquad\textrm{ where }\tilde{c}=\frac{1}{6 Ld}-\frac{1}{72 Ld^2} -\frac{1}{18 Ld^2}-\frac{1}{18 Ld }.
        \end{align*}
        In the above, $e$ is the Euler's number, and $L$ is the Lipschitz constant of $J$ as well as $\nabla J$. $S$ and $l$ are the number of iterations of the outer and inner loops of the OffP-SF-SVRG algorithm.
\end{theorem}
\begin{proof}
	See Section \ref{sec:conv}.
\end{proof}

\begin{remark}
	As mentioned earlier, SVRG has been employed in an on-policy RL context in \cite{papini2018,xu2020}. Unlike these works, we operate in an off-policy RL setting, and more importantly, use a biased gradient estimation scheme that is based on the idea of smoothed functionals.
Through a careful handling of the bias terms in several steps of the proof, we are able to obtain an order $O(1/Sl)$ bound for the OffP-SF-SVRG algorithm. The bound in \cite{papini2018} is of the form $O(1/Sl)+O(1/n)+O(1/B)$, where $B$ is the mini-batch size used for averaging in their inner-loop. In comparison, we obtain an order $O(1/Sl)$ without additional terms, and our algorithm does not require  simulation of system trajectories for mini-batching owing to the fact that we operate in the off-policy setting. In other words, the on-policy setting of \cite{papini2018,xu2020} implies $n$ system trajectories are simulated in the outer loop, while we obtain an $n$-sample average of the gradient estimate using off-policy evaluation.
Next, the bound in \cite{xu2020} is of the form $O(1/Sl)+O(1/n)$, while our bound is without the additional $O(1/n)$ term, since we can choose $n=Sl$ without requiring additional simulations.
\end{remark}
\begin{remark}
    In \cite{mgh}, the authors explore an alternative approach to variance reduction of a policy gradient algorithm in an off-policy context. The authors obtain a non-asymptotic bound of the order $O\left(1/T^{2/3}\right)$, where $T$ is the number of iterations of the policy gradient algorithm. In comparison, we obtain an improved bound of $O\left(1/T\right)$ in Theorem \ref{tm:non_asym_svrg} above.
\end{remark}

Now, we present a non-asymptotic bound for Algorithm \ref{alg:reinforce}.
For the  analysis, we rewrite the update rule in \eqref{eq:theta_update_pg} as follows:
\begin{align}
    \theta_{k+1} = \theta_k + \alpha_k \mathcal{P}_\Theta(\theta_k,\widehat{\nabla}J(\theta_k), \alpha_k),\label{eq:gd_nonasym_pg}
\end{align}
where $\mathcal{P}_\Theta(\cdot, \cdot, \cdot)$ is as defined in \eqref{eq:proxy}.

\begin{theorem}[\textbf{{OffP-REINFORCE}}]
    \label{tm:pg}
    Assume \ref{as:pol_cont}--\ref{as:proper}. Let $\mathbb{P}(R=k)=\frac{1 }{N}$, $J^*=\max_{\theta\in\Theta} J(\theta)$. Set $\alpha=\frac{1}{\sqrt{N}}$ and {$m=N$}. Then,
    \begin{align*}
        &\mathbb{E}\left[\left\lVert \mathcal{P}_\Theta(\theta_R, \nabla J (\theta_R), \alpha) \right\rVert^2\right]
        \leq \frac{\left(J^* - J(\theta_{0}) \right)}{\sqrt{N}}+\frac{L^3}{2\sqrt{N}}+\frac{2\sqrt{2\pi}e^2L^2}{\sqrt{N}}.
    \end{align*}
\end{theorem}
\begin{proof}
    See Section \ref{sec:conv}.
\end{proof}
\begin{remark}
    It is apparent that the result that we derived in Corollary \ref{cr:non_asym} is comparable to REINFORCE in an off-policy RL framework, which lets us conclude that SF-based gradient estimation is a viable alternative to the LR method.
\end{remark}

\section{Convergence analysis}
\label{sec:conv}
\subsection{Analysis of OffP-SF}
\label{sec:proofs-sf}
Our analysis proceeds through a sequence of lemmas.
We begin with a result that is well-known in the context of off-policy RL (cf. Chapter 5 of \cite{sutton_book}). We have provided the proof for the sake of completeness.
\begin{lemma}
	\label{lm:E_Jm}
 \begin{align*}
	\mathbb{E}_b\left[\hat{J}_{m}(\theta)\right] = J(\theta).
 \end{align*}
\end{lemma}
\begin{proof}
	Notice that
	\begin{align*}
    \mathbb{E}_b\left[\hat{J}_{m}(\theta)\right]
    &=\mathbb{E}_b\left[ \frac{1}{m}\sum_{j=1}^{m}\sum_{t=0}^{T^j-1}\gamma^tR_{t+1}^j \left(\prod_{i=0}^{t}\frac{\pi_{\theta}(A_i^j|S_i^j)}{b(A_i^j|S_i^j)} \right) \right] \\
    &=\frac{1}{m}\sum_{j=1}^{m}\mathbb{E}_b\left[\sum_{t=0}^{T^j-1}\gamma^t R_{t+1}^j \left(\prod_{i=0}^{t}\frac{\pi_{\theta}(A_i^j|S_i^j)}{b(A_i^j|S_i^j)} \right) \right] \\
    &= \frac{1}{m}\sum_{j=1}^{m}\mathbb{E}_{\pi_\theta}\left[\sum_{t=0}^{T^j-1}\gamma^t R_{t+1}^j\right]\\
    &=J(\theta).
\end{align*}
\end{proof}
\begin{lemma}
	\label{lm:theta'}
	$\Theta'= \{\theta':\lVert\theta' - \theta \rVert \leq 1, \theta \in \Theta\}$ is compact.
\end{lemma}
\begin{proof}
	Since $\Theta$ is compact, $ \exists \theta_c\in\Theta$, and $ r \in \R$ such that $\Theta \subseteq B(\theta_c,r)$, where $B(\theta_c,r)$ is an open ball centered at $\theta_c$ with radius $r$. The set $B[\theta_c,r+1]$ is a closed and bounded subset of $\R^d$, and hence compact. It is easy to see that $\Theta' \subseteq B[\theta_c,r+1]$. Using the fact that $\Theta$ is closed, and the definition of $\Theta'$, it is easy to see that $\Theta'$ is closed. Since every closed subset of a compact set is compact, $\Theta'$ is compact.
\end{proof}
\begin{lemma}
    \label{lm:Jm_lip}
For any $m\ge 1$, there exists a constant $L>0$ such that the following conditions hold w.p. $1$ for any $\theta_1, \theta_2 \in \Theta'$:
        \begin{align*}
            & \lvert \hat{J}_{m}(\theta_1)-\hat{J}_{m}(\theta_2)\rvert \leq L\lVert \theta_1-\theta_2\rVert,\\
            & \lVert \nabla\hat{J}_{m}(\theta_1)-\nabla\hat{J}_{m}(\theta_2)\rVert \leq L\lVert \theta_1-\theta_2\rVert.
    \end{align*}
\end{lemma}
\begin{proof}
        For any twice differentiable function $f:\mathbb{R}^n\to\mathbb{R}^{+}\setminus\{0\}$, the Hessian $\nabla^2f(\cdot)$ can be written as follows:
        \begin{align*}
            \nabla^2f(x)=f(x)\left[\nabla^2\log f(x) + \nabla\log f(x)\nabla\log f(x)^{\top} \right].
        \end{align*}
        Using the above equation and \ref{as:pol_cont}, we obtain
    \begin{align}
        \nabla^2 \prod_{i=0}^t \pi_{\theta}(A_i|S_i)
        &= \left(\prod_{i=0}^t  \pi_{\theta}(A_i|S_i)\right)\left[\nabla^2 \log \prod_{i=0}^t \pi_{\theta}(A_i|S_i) \right.\nonumber\\
        &\quad \left.+\left[\nabla \log \prod_{i=0}^t  \pi_{\theta}(A_i|S_i)\right]\left[\nabla \log \prod_{i=0}^t  \pi_{\theta}(A_i|S_i)\right]^{\top} \right] \nonumber\\
        &= \left(\prod_{i=0}^t  \pi_{\theta}(A_i|S_i)\right)\left[\sum_{i=0}^t \nabla^2 \log \pi_{\theta}(A_i|S_i) \right.\nonumber\\
        &\quad\left.+ \left[\sum_{i=0}^t \nabla \log   \pi_{\theta}(A_i|S_i) \right]\left[ \sum_{i=0}^t \nabla \log   \pi_{\theta}(A_i|S_i) \right]^{\top} \right].\label{eq:d2_pi}
    \end{align}
    From \eqref{eq:Jm}, we obtain
        \begin{align*}
            &\nabla^2\hat{J}_{m}(\theta)\\
            &=  \frac{1}{m}\sum_{j=1}^{m}\sum_{t=0}^{T^j-1}\gamma^t R_{t+1}^j \left(\prod_{i=0}^{t}\frac{1}{b(A_i^j|S_i^j)}\right)
            \nabla^2\left(\prod_{i=0}^{t}\pi_{\theta}(A_i^j|S_i^j) \right)\\
            &= \frac{1}{m}\sum_{j=1}^{m}\sum_{t=0}^{T^j-1}\gamma^tR_{t+1}^j \left(\prod_{i=0}^{t}\frac{\pi_{\theta}(A_i^j|S_i^j)}{b(A_i^j|S_i^j)}\right)\\
            &\quad\times\left[\sum_{i=0}^t \nabla^2 \log \pi_{\theta}(A_i^j|S_i^j)
            + \left[\sum_{i=0}^t \nabla \log\pi_{\theta}(A_i^j|S_i^j) \right]\left[\sum_{i=0}^t \nabla \log\pi_{\theta}(A_i^j|S_i^j) \right]^{\top} \right],
        \end{align*}
    where the last equality follows from \eqref{eq:d2_pi}. Observe that the RHS above is a sum of continuous functions, since
    $\nabla^2 \log \pi_{\theta}(\cdot|\cdot)$ is continuous w.r.t $\theta$ (see \ref{as:pol_cont}), the rewards $R_{t+1}^n$ are bounded, the policy $b$ is proper (see \ref{as:proper}), and $m$ is finite.
    Thus, $\nabla^2\hat{J}_{m}(\theta)$ is continuous which in turn implies $\nabla\hat{J}_{m}(\theta)$ is continuous. Further, since $\Theta'$ is compact, from Lemma \ref{lm:theta'}, we have
    \begin{align*}
        &\lVert\nabla^2\hat{J}_{m}(\theta)\rVert \leq \lVert\nabla^2\hat{J}_{m}(\theta)\rVert_F\leq L_1, \textrm{ and }\\
        &\lVert\nabla\hat{J}_{m}(\theta)\rVert \leq L_2, \;\forall \theta \in \Theta',
    \end{align*}
    for some constants $L_1, L_2 < \infty$. In the above, $\lVert A\rVert$ and $\lVert A\rVert_F$ denote the operator and Frobenius norm of a $d\times d$ matrix $A$.
    Let $L=\max(L_1,L_2)$. Then the result follows by Lemma~1.2.2 in \cite{nesterov_book}.
\end{proof}
\begin{lemma}
	\label{lm:J_lip}
	$J(\theta)$ and $\nabla J(\theta)$ are $L$-Lipschitz w.r.t. $\theta \in \Theta'$.
\end{lemma}
\begin{proof}
    Notice that
    \begin{align*}
        \lVert J({\theta}_1) - J({\theta}_2) \rVert
        &\stackrel{(a)}{=} \left\lVert \mathbb{E}_b\left[\hat{J}_{m}({\theta}_1)\right] - \mathbb{E}_b\left[\hat{J}_{m}({\theta}_2)\right] \right\rVert\\
        &\leq \mathbb{E}_b\left[ \left \lVert \hat{J}_{m}({\theta}_1)- \hat{J}_{m}({\theta}_2) \right \rVert \right]\\
        &\stackrel{(b)}{\leq}  L \left \lVert {\theta}_1 - {\theta}_2\right \rVert,
    \end{align*}
    where \((a)\) follows from Lemma \ref{lm:E_Jm}, and \((b)\) follows from Lemma \ref{lm:Jm_lip}.
    This proves the first claim.
        For the second claim, notice that
        \begin{align*}
            \lVert \nabla J({\theta}_1) - \nabla J({\theta}_2) \rVert
            &\stackrel{(a)}{=}\left\lVert \nabla \mathbb{E}_b \left[\hat{J}_{m}({\theta}_1)\right] - \nabla \mathbb{E}_b\left[\hat{J}_{m}({\theta}_2)\right]\right\rVert\\
            &\stackrel{(b)}{=} \left\lVert \mathbb{E}_b\left[\nabla\hat{J}_{m}({\theta}_1)\right] - \mathbb{E}_b\left[\nabla\hat{J}_{m}({\theta}_2)\right] \right\rVert\\
            &\leq \mathbb{E}_b\left[ \left \lVert \nabla \hat{J}_{m}({\theta}_1)- \nabla\hat{J}_{m}({\theta}_2) \right \rVert \right]\nonumber\\
            &\stackrel{(c)}{\leq}  L \left \lVert {\theta}_1 - {\theta}_2\right\rVert,
        \end{align*}
    where \((a)\) follows from Lemma \ref{lm:E_Jm}, and \((c)\) follows from Lemma \ref{lm:Jm_lip}.
    The equality in the step \((b)\) follows by an application of the dominated convergence theorem to interchange the differentiation and integration operations. For this application, we use the following facts:\\
    (i)
    $\mathbb{E}_b\left[\hat{J}_{m}(\theta)\right]<\infty$ holds for any $\theta \in \R^d$ because the state and actions spaces are finite, the rewards are bounded, $\pi_{\theta}(a|s)>0$ and $b(a|s) >0$, $\forall \theta \in \R^d, \forall a \in \mathscr{A}, \textrm{ and } \forall s \in \mathscr{S}$ (from \ref{as:pol_cont} and \ref{as:b_pol}), and  $\mathbb{P}\left(S_M \neq 0 \mid S_0, b \right)<1$ from \ref{as:proper};\\
    (ii) $\lVert\nabla\hat{J}_{m}(\theta)\rVert \leq L$ a.s. from Lemma \ref{lm:Jm_lip}; and\\
    (iii) $\mathbb{E}_b\left[L\right] <\infty$ since the state, as well as action spaces, are finite, and \\$\mathbb{P}\left(S_M \neq 0 \mid S_0, b \right)<1$ from \ref{as:proper}.
\end{proof}
Next, we recall a result from \cite{flaxman}, which will be used to establish the unbiasedness of the gradient estimate in \eqref{eq:nabla_Jm}.
\begin{lemma}
	\label{lm:del_J_mu_1}
 \begin{align*}
	\nabla J_{\mu}(\theta)=\mathbb{E}_{v\in\mathbb{S}^{d-1}}\left[\frac{d}{\mu}J(\theta+\mu v)v\right].
 \end{align*}
\end{lemma}
\begin{proof}
	See Lemma 2.1 in \cite{flaxman}.
\end{proof}
\begin{lemma}
	\label{lm:del_J_mu_2}
 \begin{align*}
	\mathbb{E}\left[\widehat{\nabla}_{n,\mu}\hat{J}_{m}(\theta)\right]=\nabla J_{\mu}(\theta),\; \forall 0< n <\infty.
 \end{align*}
\end{lemma}
\begin{proof}
	We follow the technique from \cite{shamir}.
	\begin{align*}
		\mathbb{E}\left[\widehat{\nabla}_{n,\mu}\hat{J}_{m}(\theta)\right]
        &= \mathbb{E}_{b}\left[\mathbb{E}_{v_{1:n}}\left[\widehat{\nabla}_{n,\mu}\hat{J}_{m}(\theta)\right] \right] \\
		&=\mathbb{E}_{b}\left[\frac{d}{n}\mathbb{E}_{v_{1:n}}\left[\sum_{i=1}^{n}\frac{\hat{J}_{m}(\theta+\mu v_i)
			- \hat{J}_{m}(\theta-\mu v_i)}{2\mu}v_i\right] \right]\\
		&\stackrel{(a)}{=}\frac{d}{2\mu} \mathbb{E}_{b}\left[\mathbb{E}_{v}\left[\left(\hat{J}_{m}(\theta+\mu v) - \hat{J}_{m}(\theta-\mu v)\right)v\right] \right] \\
		&=\frac{d}{2\mu} \mathbb{E}_{v}\left[\mathbb{E}_{b}\left[\left(\hat{J}_{m}(\theta+\mu v) - \hat{J}_{m}(\theta-\mu v)\right)v\right] \right] \\
		&\stackrel{(b)}{=}\frac{d}{2\mu} \mathbb{E}_{v}\left[\left(J(\theta+\mu v) - J(\theta-\mu v)\right)v \right]\\
		&=\frac{d}{2\mu} \left (\mathbb{E}_{v}\left[J(\theta+\mu v) v\right] + \mathbb{E}_{v}\left[J(\theta+\mu (-v)) (-v) \right]\right) \\
		&\stackrel{(c)}{=}\frac{d}{\mu} \mathbb{E}_{v}\left[J(\theta+\mu v) v\right] \\
		&\stackrel{(d)}{=}\nabla J_{\mu}(\theta),
	\end{align*}
 where \((a)\) follows since $v_{1:n}$ are i.i.d r.v.s, \((b)\) follows from Lemma \ref{lm:E_Jm}, \((c)\) follows from the fact that $v$ has a symmetric distribution, and \((d)\) follows from Lemma \ref{lm:del_J_mu_1}.
\end{proof}
The claim below bounds the bias in the gradient estimate in \eqref{eq:nabla_Jm}, and can be inferred from \cite{gao2018}. For the sake of completeness, we provide the detailed proof.
\begin{lemma}
	\label{lm:bias}
 \begin{align*}
	\left\lVert \nabla J_{\mu}(\theta) - \nabla J (\theta)\right\rVert \leq \frac{\mu dL}{2}.
 \end{align*}
\end{lemma}
\begin{proof}
Notice that
\begin{align*}
&\left\lVert \nabla J_{\mu}(\theta) - \nabla J (\theta)\right\rVert \\
&\stackrel{(a)}{=}\left\lVert \mathbb{E}_{v}\left[\frac{d}{\mu}J(\theta+\mu v)v\right]- \nabla J (\theta)\right\rVert \\
&\stackrel{(b)}{=}\left\lVert \mathbb{E}_{v}\left[\frac{d}{\mu}J(\theta+\mu v)v\right]-\frac{d}{\mu}J (\theta)\mathbb{E}_{v}\left[v\right]
 - \frac{d}{\mu}\langle \nabla J (\theta),\mu \mathbb{E}_{v}\left[vv^\top\right]\rangle\right\rVert\\
&=\frac{d}{\mu}\left\lVert \mathbb{E}_{v}\left[J(\theta+\mu v)v-J (\theta)v - \langle \nabla J (\theta),\mu v\rangle v \right]\right\rVert\\
&\leq\frac{d}{\mu}\mathbb{E}_{v}\left[\left\lVert J(\theta+\mu v) -J (\theta) - \langle \nabla J (\theta),\mu v\rangle\right\rVert \left\lVert  v \right\rVert \right]\\
&\stackrel{(c)}{\leq}\frac{d}{\mu}\mathbb{E}_{v}\left[\left\lVert J(\theta+\mu v) -J (\theta) - \langle \nabla J (\theta),\mu v\rangle\right\rVert  \right]\\
&\stackrel{(d)}{\leq}\frac{d}{\mu}\mathbb{E}_{v}\left[\frac{L}{2}\mu^2\left\lVert v\right\rVert^2  \right]\\
&\stackrel{(e)}{\leq} \frac{\mu dL}{2},
\end{align*}
where \((a)\) follows from Lemma \ref{lm:del_J_mu_1}, and \((b)\) follows since $\mathbb{E}_{v\in\mathbb{S}^{d-1}}\left[v\right]=0$ and $\mathbb{E}_{v\in\mathbb{S}^{d-1}}\left[vv^\top\right] = \frac{1}{d}\mathbbm{1}_{d \times d}$ (cf. Theorem ~2.7 in \cite{fang_book}). The step \((c)\) follows since $v\in\mathbb{S}^{d-1}$, $\left\lVert  v \right\rVert =1$, \((d)\) follows from Lemma \ref{lm:J_lip}, and \((e)\) follows since $v\in\mathbb{S}^{d-1}$, $\left\lVert  v \right\rVert =1$.
\end{proof}
{
\begin{lemma}
	\label{lm:hayes}
 Let $\{\hat{X}_i\}_{i=1}^n$ be i.i.d, vector-valued r.v.s., such that $\chi=\E\left[\hat{X}_i\right]$, $\forall \hat{X}_i$. Let $S_n=\frac{1}{n}\sum_{i=1}^n \hat{X}_i$. Assume $\forall i$, $\lVert \hat{X}_i \rVert \leq M$ a.s. Then,
 \begin{align*}
 \forall \epsilon>0,\; \p\left(\left\lVert S_n - \chi \right\rVert \geq \epsilon \right)\leq 2e^2 \exp \left( \frac{-n\epsilon^2}{8M^2}\right).
 \end{align*}
\end{lemma}
\begin{proof}
Let
\begin{align*}
Y_{n'}= \begin{cases}\frac{1}{2M}\sum\limits_{i=1}^{n'} \left(\hat{X}_i - \chi\right),& \textrm{ for }n'=\{1,\cdots,n\}\\
0,&\textrm{ for }n'=0.
\end{cases}
\end{align*}
Then $\{Y_{n'}\}_{n'=1}^{n}$ is a set of partial sums of bounded mean zero r.v.s. Hence it is a martingale, and $\forall n'>0$,
\begin{align*}
\left\lVert Y_{n'}-Y_{n'-1} \right\rVert &= \frac{1}{2M}\left\lVert\sum_{i=1}^{n'}\left(\hat{X}_i - \chi \right) - \sum_{i=1}^{n'-1}\left(\hat{X}_i - \chi \right) \right\rVert\\
&= \frac{1}{2M}\left\lVert\hat{X}_{n'} - \chi \right\rVert\\
&\leq \frac{1}{2M} 2M=1.
\end{align*}
Now,
\begin{align*}
\p\left(\left\lVert S_n - \chi \right\rVert \geq \epsilon\right)&=\p\left(\left\lVert Y_n\right\rVert\geq \frac{n\epsilon}{2M} \right)\\
&\stackrel{(a)}{\leq} 2e^2 \exp \left( \frac{-n\epsilon^2}{8M^2}\right).
\end{align*}
In the above, the step $(a)$ follows from \cite[Theorem 1.8]{hayes2005large}, and the fact that every martingale is a very weak martingale (cf. \cite[Definition 1.3]{hayes2005large}).
\end{proof}
\begin{lemma}
	\label{lm:bias_err}
 \begin{align*}
 \mathbb{E}\left [\left\lVert\widehat{\nabla}_{n,\mu} \hat{J}_{m}(\theta) - \nabla J_{\mu}(\theta) \right \rVert \right ] \leq \frac{2\sqrt{2\pi}e^2dL}{\sqrt{n}}.
 \end{align*}
\end{lemma}
\begin{proof}
	From \eqref{eq:nabla_Jm}, we obtain
  \begin{align}
 \label{eq:hat_nabla_Jm_bound}
 	\left\lVert\widehat{\nabla}_{n,\mu} \hat{J}_{m}(\theta)\right\rVert
    &=\left\lVert\frac{d}{n}\sum_{i=1}^{n}\frac{\hat{J}_{m}(\theta+\mu v_i) - \hat{J}_{m}(\theta - \mu v_i)}{2\mu}v_i\right\rVert\nonumber\\
    &\stackrel{(a)}{\leq}\frac{d}{2\mu n}\sum_{i=1}^{n}\left\lVert\left(\hat{J}_{m}(\theta+\mu v_i) - \hat{J}_{m}(\theta - \mu v_i)\right)v_i\right\rVert\nonumber\\
    &\stackrel{(b)}{=}\frac{d}{2\mu n}\sum_{i=1}^{n}\left\lvert\hat{J}_{m}(\theta+\mu v_i) - \hat{J}_{m}(\theta - \mu v_i)\right\rvert\left\lVert v_i\right\rVert\nonumber\\
    &\stackrel{(c)}{=}\frac{d}{2\mu n}\sum_{i=1}^{n}\left\lvert\hat{J}_{m}(\theta+\mu v_i) - \hat{J}_{m}(\theta - \mu v_i)\right\rvert\nonumber\\
    &\stackrel{(d)}{\leq}\frac{d}{2\mu n}\sum_{i=1}^{n}L\left\lVert 2\mu v_i\right\rVert\nonumber\\
     &\stackrel{(f)}{\leq}dL \textrm{ a.s.},\; \forall 0<n<\infty.
 \end{align}
 In the above, the step $(a)$ follows from the fact that $\lVert\sum_{i=1}^n a_i \rVert \leq \sum_{i=1}^n \lVert a_i \rVert$, and step $(b)$ follows from the fact that for a scalar a and a vector B, $\lVert aB \rVert = \lvert a \rvert \lVert B \rVert$. The steps $(c)$ and $(f)$ follow since $\left\lVert v_i\right\rVert=1$. The step $(d)$ follows from Lemma \ref{lm:Jm_lip}.

From \eqref{eq:hat_nabla_Jm_bound}, Lemma \ref{lm:del_J_mu_2}, and Lemma \ref{lm:hayes}, we obtain
    \begin{align}
        \label{eq:hat_nabla_Jm_prob}
        \p\left(\left\lVert \widehat{\nabla}_{\mu,n} \hat{J}_m({\theta})  - \nabla J_{\mu}(\theta)  \right\rVert > \epsilon\right) \leq 2e^2\exp\left(\frac{-n\epsilon^2}{8d^2 L^2}\right),
    \end{align}
and
    \begin{align}
\E\left[\left\lVert \widehat{\nabla}_{\mu,n} \hat{J}_m({\theta})  - \nabla J_{\mu}(\theta) \right\rVert \right]
&= \int_{0}^{\infty} \p\left( \left\lVert \widehat{\nabla}_{\mu,n} \hat{J}_m({\theta})  - \nabla J_{\mu}(\theta)\right\rVert > \epsilon\right)d\epsilon \nonumber\\
&\stackrel{(a)}{\leq} \int_{0}^{\infty} 2e^2\exp\left(-\frac{n{\epsilon}^2}{8d^2L^2}\right) d\epsilon \nonumber\\
&\stackrel{(b)}{=} \frac{2\sqrt{2\pi}e^2dL}{\sqrt{n}}.\label{eq:hat_nabla_Jm_err}
\end{align}
In the above, the step $(a)$ follows from \eqref{eq:hat_nabla_Jm_prob}, and step $(b)$ follows from the fact that $\int_{0}^{\infty}exp(-a\epsilon^2)d\epsilon=\frac{1}{2}\sqrt{\frac{\pi}{a}},\;\forall a>0$.
\end{proof}
\begin{lemma}
	\label{lm:var}
 \begin{align*}
	\E\left[\left\lVert \widehat{\nabla}_{n,\mu} \hat{J}_{m}(\theta) \right\rVert^2\right] \leq d^2L^2.
 \end{align*}
\end{lemma}
\begin{proof}
 \begin{align}
 \label{eq:hat_nabla_Jm_var}
 	\E\left[\left\lVert\widehat{\nabla}_{n,\mu} \hat{J}_{m}(\theta)\right\rVert^2 \right]
    &=\E\left[\left\lVert\frac{d}{n}\sum_{i=1}^{n}\frac{\hat{J}_{m}(\theta+\mu v_i) - \hat{J}_{m}(\theta - \mu v_i)}{2\mu}v_i\right\rVert^2\right]\nonumber\\
    &\stackrel{(a)}{\leq}\frac{d^2}{4\mu^2 n^2}\E\left[n\sum_{i=1}^{n}\left\lVert\left(\hat{J}_{m}(\theta+\mu v_i) - \hat{J}_{m}(\theta - \mu v_i)\right)v_i\right\rVert^2\right]\nonumber\\
    &\stackrel{(b)}{=}\frac{d^2}{4\mu^2}\E\left[\left\lVert\left(\hat{J}_{m}(\theta+\mu v) - \hat{J}_{m}(\theta - \mu v)\right)v\right\rVert^2\right]\nonumber\\
    &\stackrel{(c)}{=}\frac{d^2}{4\mu^2 }\E\left[\left\lvert\hat{J}_{m}(\theta+\mu v) - \hat{J}_{m}(\theta - \mu v)\right\rvert^2\left\lVert v\right\rVert^2\right]\nonumber\\
    &\stackrel{(d)}{=}\frac{d^2}{4\mu^2 }\E\left[\left\lvert\hat{J}_{m}(\theta+\mu v) - \hat{J}_{m}(\theta - \mu v)\right\rvert^2\right]\nonumber\\
    &\stackrel{(e)}{\leq}\frac{d^2}{4\mu^2 }\E\left[4 L^2\mu^2 \left\lVert v\right\rVert^2\right]\nonumber\\
     &\stackrel{(f)}{\leq}d^2L^2.
 \end{align}
  In the above, the step $(a)$ follows from the fact that $\lVert\sum_{i=1}^n a_i \rVert^2 \leq n\sum_{i=1}^n \lVert a_i \rVert^2 $, and step $(b)$ follows since $v_{1:n}$ are i.i.d r.v.s. The step $(c)$ follows from the fact that for a scalar $a$ and a vector $B$, $\lVert aB\rVert^2=\lvert a\rvert^2 \lVert B\rVert^2 $. The step $(d)$ and $(f)$ follow since $\left\lVert v\right\rVert=1$. The step $(e)$ follows from Lemma \ref{lm:Jm_lip}.
\end{proof}
}
The claim below is well-known in the context of projections on to convex sets. We have provided the proof for the sake of completeness.
\begin{lemma}
	\label{lm:proxy}
	The projection operator $\mathcal{P}_\Theta$ defined in \eqref{eq:proxy} satisfies
	\begin{align*}
		&(\romannumeral1) \left\lVert \mathcal{P}_\Theta(\theta, f(\theta), \alpha) \right\rVert \leq  \left\lVert f(\theta) \right\rVert,\\
		&(\romannumeral2) \left\lVert \mathcal{P}_\Theta(\theta, f(\theta), \alpha) - \mathcal{P}_\Theta(\theta, g(\theta), \alpha) \right\rVert
		 \leq \left\lVert f(\theta) - g(\theta) \right\rVert, \textrm{ and}\\
		&(\romannumeral3) \left\langle f(\theta),\mathcal{P}_\Theta(\theta, f(\theta), \alpha) \right\rangle \geq  \left\lVert \mathcal{P}_\Theta(\theta, f(\theta), \alpha)  \right\rVert^2.
	\end{align*}
\end{lemma}
\begin{proof}
$(\romannumeral1)$
\begin{align*}
	\left\lVert \mathcal{P}_\Theta(\theta, f(\theta), \alpha) \right\rVert
	&\stackrel{(a)}{=}  \frac{1}{\alpha}\left\lVert \Pi_{\Theta}(\theta+\alpha f(\theta))-\theta\right\rVert \\
	&\stackrel{(b)}{\leq}  \frac{1}{\alpha}\left\lVert \theta+\alpha f(\theta) -\theta\right\rVert \\
    &= \left\lVert f(\theta) \right\rVert,
\end{align*}
where \((a)\) follows from \eqref{eq:proxy}, and \((b)\) follows since $\left\lVert \Pi_{\Theta}(x) - y \right\rVert \leq \left\lVert x - y \right\rVert,\,\forall y \in \Theta$.

$(\romannumeral2)$
	\begin{align*}
	&\left\lVert \mathcal{P}_\Theta(\theta, f(\theta), \alpha) - \mathcal{P}_\Theta(\theta, g(\theta), \alpha)\right\rVert  \\
	&\stackrel{(a)}{=}  \left\lVert \frac{1}{\alpha}\left[\Pi_{\Theta}(\theta+\alpha f(\theta))-\theta\right]
	- \frac{1}{\alpha}\left[\Pi_{\Theta}(\theta+\alpha g(\theta))-\theta\right]\right\rVert \\
	&= \frac{1}{\alpha} \left\lVert \Pi_{\Theta}(\theta+\alpha f(\theta)) - \Pi_{\Theta}(\theta+\alpha g(\theta))\right\rVert \\
	&\stackrel{(b)}{\leq}  \frac{1}{\alpha}\left\lVert \theta+\alpha f(\theta) -\theta-\alpha g(\theta)\right\rVert \\
	&\leq  \left\lVert f(\theta) - g(\theta) \right\rVert,
	\end{align*}
 where \((a)\) follows from \eqref{eq:proxy}, and \((b)\) follows since\\ $\left\lVert \Pi_{\Theta}(x) - \Pi_{\Theta}(y) \right\rVert \leq \left\lVert x - y \right\rVert,\,\forall x,y$.

$(\romannumeral3)$
	\begin{align*}
	&\left\langle f(\theta),\mathcal{P}_\Theta(\theta, f(\theta), \alpha) \right\rangle -  \left\lVert \mathcal{P}_\Theta(\theta, f(\theta), \alpha)  \right\rVert^2 \\
	&=\left\langle f(\theta),\mathcal{P}_\Theta(\theta, f(\theta), \alpha) \right\rangle
	-\left\langle \mathcal{P}_\Theta(\theta, f(\theta), \alpha) ,\mathcal{P}_\Theta(\theta, f(\theta), \alpha) \right\rangle\\
	&=\left\langle f(\theta)-\mathcal{P}_\Theta(\theta, f(\theta),\alpha),\mathcal{P}_\Theta(\theta, f(\theta), \alpha) \right\rangle\\
	&=\left\langle f(\theta)-\frac{1}{\alpha}\left[\Pi_\Theta(\theta+\alpha f(\theta))-\theta\right],
	\frac{1}{\alpha}\left[\Pi_\Theta(\theta+\alpha f(\theta))-\theta\right] \right\rangle\\
	&\stackrel{(a)}{=}-\frac{1}{\alpha^2}\left\langle \Pi_\Theta(\theta+\alpha f(\theta))-(\theta + \alpha f(\theta)),
	\Pi_\Theta(\theta+\alpha f(\theta))-\theta \right\rangle \geq 0,
	\end{align*}
 where \((a)\) follows since $\left\langle \Pi_\Theta(x)-x, \Pi_\Theta(x)-y\right\rangle \leq 0,\,\forall y \in \Theta$.
\end{proof}

\subsection*{Proof of Theorem \ref{tm:non_asym}}
	Using the fundamental theorem of calculus, we obtain
	\begin{align}
		& J(\theta_k) - J(\theta_{k+1}) \nonumber\\
		&=\langle \nabla J(\theta_k), \theta_k - \theta_{k+1} \rangle
		+ \int_0^1 \left \langle \nabla J(\theta_{k+1}+\tau(\theta_k - \theta_{k+1}))-\nabla J(\theta_k),\theta_k - \theta_{k+1} \right \rangle d\tau \nonumber\\
		&\stackrel{(a)}{\leq}\langle \nabla J(\theta_k), \theta_k - \theta_{k+1} \rangle
		+ \int_0^1 \left\lVert \nabla J(\theta_{k+1}+\tau(\theta_k - \theta_{k+1}))-\nabla J(\theta_k)\right\rVert
		\left\lVert \theta_k - \theta_{k+1} \right \rVert d\tau \nonumber\\
		&\stackrel{(b)}{\leq} \left \langle \nabla J(\theta_k), \theta_k - \theta_{k+1} \right \rangle
		  + L\left\lVert \theta_k - \theta_{k+1} \right\rVert^2  \int_0^1 (1-\tau) d\tau \nonumber\\
		&\leq \left \langle \nabla J(\theta_k), \theta_k - \theta_{k+1} \right \rangle + \frac{L}{2}\left\lVert \theta_k - \theta_{k+1} \right\rVert^2\\
		&\stackrel{(c)}{\leq} \alpha_k \left \langle \nabla J(\theta_k), -\mathcal{P}_\Theta(\theta_k,\widehat{\nabla}_{n_k,\mu_k} \hat{J}_{m_k}(\theta_k), \alpha_k) \right \rangle
		  + \frac{L\alpha_k^2}{2}\left\lVert \mathcal{P}_\Theta(\theta_k,\widehat{\nabla}_{n_k,\mu_k} \hat{J}_{m_k}(\theta_k), \alpha_k) \right\rVert^2 \nonumber\\
		&\leq \alpha_k \left \langle \nabla J(\theta_k), \mathcal{P}_\Theta(\theta_k,\nabla J (\theta_k), \alpha_k)
		-\mathcal{P}_\Theta(\theta_k,\widehat{\nabla}_{n_k,\mu_k} \hat{J}_{m_k}(\theta_k), \alpha_k)\right \rangle \nonumber\\
		&\quad -\alpha_k \left \langle \nabla J(\theta_k), \mathcal{P}_\Theta(\theta_k,\nabla J (\theta_k), \alpha_k) \right \rangle
		+ \frac{L\alpha_k^2}{2}\left\lVert \mathcal{P}_\Theta(\theta_k,\widehat{\nabla}_{n_k,\mu_k} \hat{J}_{m_k}(\theta_k), \alpha_k) \right\rVert^2 \nonumber\\
		&\stackrel{(d)}{\leq} \alpha_k \left \lVert \nabla J(\theta_k) \right \rVert\left\lVert \nabla J (\theta_k) -\widehat{\nabla}_{n_k,\mu_k} \hat{J}_{m_k}(\theta_k)\right \rVert \nonumber\\
		&\quad -\alpha_k \left \lVert \mathcal{P}_\Theta(\theta_k,\nabla J(\theta_k), \alpha_k) \right \rVert^2
		  + \frac{L\alpha_k^2}{2}\left\lVert \widehat{\nabla}_{n_k,\mu_k} \hat{J}_{m_k}(\theta_k) \right\rVert^2 \nonumber\\
		&\stackrel{(e)}{\leq} L\alpha_k \left\lVert \nabla J (\theta_k) - \widehat{\nabla}_{n_k,\mu_k} \hat{J}_{m_k}(\theta_k)\right \rVert \nonumber\\
		&\quad -\alpha_k \left \lVert \mathcal{P}_\Theta(\theta_k,\nabla J(\theta_k), \alpha_k) \right \rVert^2
		  + \frac{L\alpha_k^2}{2}\left\lVert \widehat{\nabla}_{n_k,\mu_k} \hat{J}_{m_k}(\theta_k)\right\rVert^2 \nonumber\\
		&\leq L\alpha_k \left\lVert \nabla J (\theta_k) - \nabla J_{\mu_k} (\theta_k) \right \rVert
		+L\alpha_k \left\lVert \nabla J_{\mu_k} (\theta_k)-\widehat{\nabla}_{n_k,\mu_k} \hat{J}_{m_k}(\theta_k)\right \rVert \nonumber\\
		&\quad -\alpha_k \left \lVert \mathcal{P}_\Theta(\theta_k,\nabla J(\theta_k), \alpha_k) \right \rVert^2
		  + \frac{L\alpha_k^2}{2}\left\lVert \widehat{\nabla}_{n_k,\mu_k} \hat{J}_{m_k}(\theta_k)\right\rVert^2 \nonumber\\
		&\stackrel{(f)}{\leq} \frac{dL^2}{2}\alpha_k \mu_k
		+L\alpha_k \left\lVert \nabla J_{\mu_k}(\theta_k)-\widehat{\nabla}_{n_k,\mu_k} \hat{J}_{m_k}(\theta_k)\right \rVert\nonumber\\
		& \quad -\alpha_k \left \lVert \mathcal{P}_\Theta(\theta_k,\nabla J(\theta_k), \alpha_k) \right \rVert^2
		+ \frac{L\alpha_k^2}{2}\left\lVert \widehat{\nabla}_{n_k,\mu_k} \hat{J}_{m_k}(\theta_k) \right\rVert^2, \label{eq:1}
	\end{align}
where \((a)\) follows from the Cauchy–Schwarz inequality, and \((b)\) follows from Lemma \ref{lm:J_lip}. The step \((c)\) follows from \eqref{eq:gd_nonasym}, and \((d)\) follows from Lemma \ref{lm:proxy}. The step \((e)\) follows from Lemma \ref{lm:J_lip}, and \((f)\) follows from Lemma \ref{lm:bias}.

	Summing \eqref{eq:1} for $k=0,\ldots,N-1$, we obtain
	\begin{align}
		&\sum_{k=0}^{N-1}\alpha_k \left \lVert \mathcal{P}_\Theta(\theta_k,\nabla J(\theta_k), \alpha_k) \right \rVert^2 \nonumber\\
		&\leq \left(J(\theta_N) - J(\theta_{0}) \right)+ \frac{dL^2}{2}\sum_{k=0}^{N-1} \alpha_k \mu_k\nonumber\\
		&\quad+L\sum_{k=0}^{N-1} \alpha_k \left \lVert \nabla J_{\mu_k}(\theta_k)-\widehat{\nabla}_{n_k,\mu_k} \hat{J}_{m_k}(\theta_k) \right\rVert
		  +\frac{L}{2}\sum_{k=0}^{N-1}\alpha_k^2\left\lVert \widehat{\nabla}_{n_k,\mu_k} \hat{J}_{m_k}(\theta_k)\right\rVert^2\nonumber\\
		&\leq \left(J^* - J(\theta_{0}) \right)+ \frac{dL^2}{2}\sum_{k=0}^{N-1} \alpha_k \mu_k\nonumber\\
		&\quad+L\sum_{k=0}^{N-1} \alpha_k \left \lVert \nabla J_{\mu_k}(\theta_k)-\widehat{\nabla}_{n_k,\mu_k} \hat{J}_{m_k}(\theta_k) \right\rVert
		+\frac{L}{2}\sum_{k=0}^{N-1}\alpha_k^2\left\lVert \widehat{\nabla}_{n_k,\mu_k} \hat{J}_{m_k}(\theta_k)\right\rVert^2.\label{eq:2}
	\end{align}
	Taking expectations on both sides of \eqref{eq:2}, we obtain
	\begin{align*}
		&\sum_{k=0}^{N-1}\alpha_k \mathbb{E}\left[\left \lVert \mathcal{P}_\Theta(\theta_k,\nabla J(\theta_k), \alpha_k) \right \rVert^2\right ] \nonumber\\
		&\leq \left(J^* - J(\theta_{0}) \right)+ \frac{dL^2}{2}\sum_{k=0}^{N-1} \alpha_k \mu_k\nonumber\\
		&\quad+L\sum_{k=0}^{N-1} \alpha_k\mathbb{E}\left[ \left \lVert \nabla J_{\mu_k}(\theta_k)-\widehat{\nabla}_{n_k,\mu_k} \hat{J}_{m_k}(\theta_k) \right\rVert\right]
		+\frac{L}{2}\sum_{k=0}^{N-1}\alpha_k^2\mathbb{E}\left[\left\lVert \widehat{\nabla}_{n_k,\mu_k} \hat{J}_{m_k}(\theta_k)\right\rVert^2\right]\nonumber\\
        &\stackrel{(a)}{\leq} \left(J^* - J(\theta_{0}) \right)+\frac{dL^2}{2}\sum_{k=0}^{N-1} \alpha_k \mu_k + 2\sqrt{2\pi}e^2dL^2 \sum_{k=0}^{N-1} \frac{\alpha_k}{\sqrt{n_k}} +\frac{d^2L^3}{2}\sum_{k=0}^{N-1}\alpha_k^2,
	\end{align*}
where \((a)\) follows from Lemmas \ref{lm:bias_err} and \ref{lm:var}.

	Since $\mathbb{P}(R=k)=\frac{\alpha_k }{\sum_{k=0}^{N-1}\alpha_k}$, we obtain
	\begin{align*}
		&\mathbb{E}\left[\left\lVert \mathcal{P}_\Theta(\theta_R, \nabla J (\theta_R), \alpha_R) \right\rVert^2\right]\\
		&= \frac{\sum\limits_{k=0}^{N-1}\alpha_k\mathbb{E}\left[\left\lVert \mathcal{P}_\Theta(\theta_k,\nabla J(\theta_k), \alpha_k) \right\rVert^2\right]}{\sum\limits_{k=0}^{N-1}\alpha_k}\\
 &\stackrel{(a)}{\leq} \frac{\left(J^* - J(\theta_{0}) \right)+\frac{dL^2}{2}\sum_{k=0}^{N-1} \alpha_k \mu_k + 2\sqrt{2\pi}e^2dL^2 \sum_{k=0}^{N-1} \frac{\alpha_k}{\sqrt{n_k}} +\frac{d^2L^3}{2}\sum_{k=0}^{N-1}\alpha_k^2}{\sum\limits_{k=0}^{N-1}\alpha_k}.
	\end{align*}
 \hfill{\qed}
\subsection*{Proof of Corollary \ref{cr:non_asym}}
{
	In \eqref{eq:nonasym-gen-bd}, we substitute $\alpha_k=\frac{1}{\sqrt{N}}$, $\mu_k=\frac{1}{\sqrt{N}}$, and $n_k=N$, $\forall k$, to obtain
	\begin{align*}
		&\E\left[\left\lVert \mathcal{P}_\Theta(\theta_R, \nabla J (\theta_R), \alpha_R) \right\rVert^2\right]
        \leq\frac{J^* - J(\theta_{0})}{\sqrt{N}}+\frac{dL^2+d^2L^3}{2\sqrt{N}} + \frac{2\sqrt{2\pi}e^2dL^2 }{\sqrt{N}}.
	\end{align*}
 \hfill{\qed}
}
\subsection{Analysis of OffP-SF-SVRG}
\label{sec:proofs-svrg}
\begin{lemma}
	\label{lm:E_Jhat}
 \begin{align*}
	\mathbb{E}_b\left[\hat{J}^j(\theta)\right] = J(\theta),\;\forall j.
 \end{align*}
\end{lemma}
\begin{proof}
	Notice that
    \begin{align*}
        \mathbb{E}_b\left[\hat{J}^j(\theta)\right]
        &=\mathop{\mathbb{E}}_{\substack{[1, m] \sim b \\ j \in[1,m]}}\left[\sum_{t=0}^{T^j-1}\gamma^t R_{t+1}^j \left(\prod_{i=0}^{t}\frac{\pi_{\theta}(A_i^j |S_i^j)}{b(A_i^j|S_i^j)} \right) \right] \\
        &=\mathop{\mathbb{E}}_{\substack{[1,m] \sim {\pi_\theta} \\ j \in[1,m]}}\left[\sum_{t=0}^{T^j-1}\gamma^t R_{t+1}^j\right]
        =J(\theta).
    \end{align*}
\end{proof}
\begin{lemma}
	\label{lm:Jm_EJhat}
 \begin{align*}
	\widehat{\nabla}_{n,\mu} \hat{J}_{m}(\theta) =\mathop{\mathbb{E}}_{j\in[1,m]}\left[\widehat{\nabla}_{n,\mu} \hat{J}^j(\theta) \right].
 \end{align*}
\end{lemma}
\begin{proof}
	Notice that
	\begin{align*}
		\widehat{\nabla}_{n,\mu} \hat{J}_{m}(\theta)
		&=\frac{d}{n}\sum_{i=1}^{n}\frac{\hat{J}_{m}(\theta+\mu v_i) - \hat{J}_{m}(\theta - \mu v_i)}{2\mu}v_i\\
		&\stackrel{(a)}{=}\frac{1}{m}\sum_{j=1}^{m}\frac{d}{n}\sum_{i=1}^{n}\frac{\hat{J}^j(\theta+\mu v_i) - \hat{J}^j(\theta - \mu v_i)}{2\mu}v_i\\
		&=\mathop{\mathbb{E}}_{j\in[1,m]}\left[\frac{d}{n}\sum_{i=1}^{n} \frac{\hat{J}^j(\theta+\mu v_i) - \hat{J}^j(\theta - \mu v_i)}{2\mu}v_i\right]\\
		&\stackrel{(b)}{=}\mathop{\mathbb{E}}_{j\in[1,m]}\left[\widehat{\nabla}_{n,\mu} \hat{J}^j(\theta) \right],
	\end{align*}
 where \((a)\) follows from \eqref{eq:Jm} and \eqref{eq:Jhat_j}. The step \((b)\) follows from \eqref{eq:nabla_Jhat_j}.
\end{proof}
{
\begin{lemma}
	\label{lm:del_Jhatj_bias}
 \begin{align*}
	\mathbb{E}\left[\widehat{\nabla}_{n,\mu}\hat{J}^j(\theta)\right]=\nabla J_{\mu}(\theta),\;\forall 0<n<\infty.
 \end{align*}
\end{lemma}
\begin{proof}
	We follow the technique from \cite{shamir}.
	\begin{align*}
		\mathbb{E}\left[\widehat{\nabla}_{n,\mu}\hat{J}^{j}(\theta)\right]
        &= \mathbb{E}_{b}\left[\mathbb{E}_{v_{1:n}}\left[\widehat{\nabla}_{n,\mu}\hat{J}^{j}(\theta)\right] \right] \\
		&=\mathbb{E}_{b}\left[\frac{d}{n}\mathbb{E}_{v_{1:n}}\left[\sum_{i=1}^{n}\frac{\hat{J}^{j}(\theta+\mu v_i)
			- \hat{J}^{j}(\theta-\mu v_i)}{2\mu}v_i\right] \right]\\
		&\stackrel{(a)}{=}\frac{d}{2\mu} \mathbb{E}_{b}\left[\mathbb{E}_{v}\left[\left(\hat{J}^{j}(\theta+\mu v) - \hat{J}^{j}(\theta-\mu v)\right)v\right] \right] \\
		&=\frac{d}{2\mu} \mathbb{E}_{v}\left[\mathbb{E}_{b}\left[\left(\hat{J}^{j}(\theta+\mu v) - \hat{J}^{j}(\theta-\mu v)\right)v\right] \right] \\
		&\stackrel{(b)}{=}\frac{d}{2\mu} \mathbb{E}_{v}\left[\left(J(\theta+\mu v) - J(\theta-\mu v)\right)v \right]\\
		&=\frac{d}{2\mu} \mathbb{E}_{v}\left[J(\theta+\mu v) v\right] + \mathbb{E}_{v}\left[J(\theta+\mu (-v)) (-v) \right] \\
		&\stackrel{(c)}{=}\frac{d}{\mu} \mathbb{E}_{v}\left[J(\theta+\mu v) v\right] \\
		&\stackrel{(d)}{=}\nabla J_{\mu}(\theta),\forall 0<n<\infty.
	\end{align*}
 In the above, the step \((a)\) follows since $v_{1:n}$ are i.i.d r.v.s, the step \((b)\) follows from Lemma \ref{lm:E_Jhat}, the step \((c)\) follows from the fact that $v$ has a symmetric distribution, and the step \((d)\) follows from Lemma \ref{lm:del_J_mu_1}.
\end{proof}
\begin{lemma}
	\label{lm:Eg_Jmu}
 \begin{align*}
	\mathbb{E}\left[g_k^s\right]=\nabla J_{\mu}(\theta_k^s).
 \end{align*}
\end{lemma}
\begin{proof}
	Notice that
	\begin{align*}
		\mathbb{E}\left[g_k^s\right]
		&=\mathbb{E}\left[\widehat{\nabla}_{n,\mu} \hat{J}^j(\theta_k^s) - \widehat{\nabla}_{n,\mu} \hat{J}^j(\tilde{\theta}^s) + \widehat{\nabla}_{n,\mu}\hat{J}_{m}(\tilde{\theta}^s)\right]\\
		&=\mathbb{E}\left[\widehat{\nabla}_{n,\mu} \hat{J}^j(\theta_k^s)\right]
		+\mathbb{E}\left[ \mathop{\mathbb{E}}_{j\in[1,m]}\left[\widehat{\nabla}_{n,\mu}\hat{J}_{m}(\tilde{\theta}^s)-\widehat{\nabla}_{n,\mu} \hat{J}^j(\tilde{\theta}^s)\right]\right]\\
		&\stackrel{(a)}{=}\mathbb{E}\left[\widehat{\nabla}_{n,\mu} \hat{J}^j(\theta_k^s)\right]\\
  		&\stackrel{(b)}{=} \nabla J_{\mu}(\theta_k^s).
	\end{align*}
 where the step $(a)$ follows from Lemma \ref{lm:Jm_EJhat} and the step $(b)$ follows from Lemma \ref{lm:del_Jhatj_bias}.
\end{proof}
\begin{lemma}
	\label{lm:Jmu_J}
 \begin{align*}
	\left\lvert J_{\mu}(\theta) - J (\theta)\right\rvert \leq \frac{ L\mu^2}{2}.
 \end{align*}
\end{lemma}
\begin{proof}
	From \eqref{eq:J_mu}, we obtain
\begin{align*}
\left\lvert J_{\mu}(\theta) - J (\theta)\right\rvert
&=\left\lvert \E_{u \in \mathbb{B}^d}\left[J(\theta+\mu u)\right] - J (\theta)\right\rvert \\
&\stackrel{(a)}{=}\left\lvert \E_{u \in \mathbb{B}^d}\left[J(\theta+\mu u) - J (\theta)-\left\langle\nabla J (\theta) ,\mu u\right\rangle\right]\right\rvert\\
&\leq \E_{u \in \mathbb{B}^d}\left[\left\lvert J(\theta+\mu u) - J (\theta)-\left\langle \nabla J (\theta) ,\mu u\right\rangle\right\rvert\right]\\
&\stackrel{(b)}{\leq} \E_{u \in \mathbb{B}^d}\left[\frac{L}{2}\left\lVert \mu u\right\rVert^2\right]\\
&\stackrel{(c)}{\leq}\frac{L\mu^2}{2}.
\end{align*}
In the above, the step $(a)$ follows since $\E_{u\in\mathbb{B}^{d}}\left[u\right]=0$. The step $(b)$ follows from \cite[Lemma 1.2.3]{nesterov_book}. The step $(c)$ follows since $\lVert u \rVert \leq 1, \forall u\in \mathbb{B}^d$.
\end{proof}

\begin{lemma}
\label{lm:nabla_Jmu_lip}
 \begin{align*}
 \left\lVert \nabla J_\mu(\theta_1) - \nabla J_\mu(\theta_2) \right\rVert \leq L\left\lVert \theta_1 - \theta_2 \right\rVert.
 \end{align*}
\end{lemma}
\begin{proof}
Since $\left\lVert \nabla J(\theta) \right\rVert \leq L$ from Lemma \ref{lm:J_lip}, we interchange the expectation and derivative by utilizing dominated convergence theorem to derive an expression for $\nabla J_\mu(\theta)$ as given below:
    \begin{align}
    \label{eq:nabla_J_mu}
    	\nabla J_\mu(\theta) = \nabla \mathbb{E}_{u \in \mathbb{B}^d}\left[J(\theta+\mu u)\right]= \mathbb{E}_{u \in \mathbb{B}^d}\left[\nabla J(\theta+\mu u)\right].
    \end{align}
From \eqref{eq:nabla_J_mu}, we obtain
    \begin{align}
    \label{eq:nabla_J_mu_lip}
    \left\lVert \nabla J_\mu(\theta_1) - \nabla J_\mu(\theta_2) \right\rVert
    &=  \left\lVert \mathbb{E}_{u \in \mathbb{B}^d}\left[\nabla J(\theta_1+\mu u)-\nabla J(\theta_2+\mu u)\right]\right\rVert \nonumber\\
    & \stackrel{(a)}{\leq}   \mathbb{E}_{u \in \mathbb{B}^d}\left[\left\lVert\nabla J(\theta_1+\mu u)-\nabla J(\theta_2+\mu u)\right\rVert\right]\nonumber\\
    & \stackrel{(b)}{\leq}   \mathbb{E}_{u \in \mathbb{B}^d}\left[L\left\lVert\theta_1-\theta_2 \right\rVert\right]\nonumber\\
    & =  L\left\lVert\theta_1-\theta_2 \right\rVert.
    \end{align}
    In the above, the step $(a)$ follows from the fact that $\lVert E[X] \rVert \leq  E[\lVert X \rVert]$, and the step $(b)$ follows from Lemma \ref{lm:J_lip}.
\end{proof}
\begin{lemma}
	\label{lm:bias_err_sq_jm}
 \begin{align*}
 \mathbb{E}\left [\left\lVert\widehat{\nabla}_{n,\mu} \hat{J}_{m}(\theta) - \nabla J_{\mu}(\theta) \right \rVert^2 \right ] \leq \frac{16e^2d^2L^2}{n}.
 \end{align*}
\end{lemma}
\begin{proof}
From \eqref{eq:hat_nabla_Jm_bound} in Lemma \ref{lm:bias_err} we have $\forall 0<n<\infty$, $\left\lVert\widehat{\nabla}_{n,\mu} \hat{J}_{m}(\theta)\right\rVert \leq dL$ a.s., and from Lemma \ref{lm:del_J_mu_2}, we have $\E\left[\widehat{\nabla}_{\mu,n} \hat{J}_m({\theta})\right]=\nabla J_{\mu}(\theta)$.

Hence, from Lemma \ref{lm:hayes}, we obtain
    \begin{align}
        \label{eq:hat_nabla_Jm_prob2}
        \p\left(\left\lVert \widehat{\nabla}_{\mu,n} \hat{J}_m({\theta})  - \nabla J_{\mu}(\theta)  \right\rVert > \epsilon\right) \leq 2e^2\exp\left(\frac{-n\epsilon^2}{8d^2 L^2}\right),
    \end{align}
and
    \begin{align*}
\E\left[\left\lVert \widehat{\nabla}_{\mu,n} \hat{J}_m({\theta})  - \nabla J_{\mu}(\theta) \right\rVert^2 \right]
&= \int_{0}^{\infty} \p\left( \left\lVert \widehat{\nabla}_{\mu,n} \hat{J}_m({\theta})  - \nabla J_{\mu}(\theta)\right\rVert > \sqrt{\epsilon}\right)d\epsilon \nonumber\\
&\stackrel{(a)}{\leq} \int_{0}^{\infty} 2e^2\exp\left(-\frac{n{\epsilon}}{8d^2L^2}\right) d\epsilon \nonumber\\
&\stackrel{(b)}{=} \frac{16e^2d^2L^2}{n}.
\end{align*}
In the above, step $(a)$ follows from \eqref{eq:hat_nabla_Jm_prob2}, and step $(b)$ follows from the fact that $\int_{0}^{\infty}\exp(-a\epsilon)d\epsilon=\frac{1}{a},\;\forall a>0$.
\end{proof}
\begin{lemma}
	\label{lm:bias_err_sq_jj}
 \begin{align*}
 \mathbb{E}\left [\left\lVert\widehat{\nabla}_{n,\mu} \hat{J}^{j}(\theta) - \nabla J_{\mu}(\theta) \right \rVert^2 \right ] \leq \frac{16e^2d^2L^2}{n}.
 \end{align*}
\end{lemma}
\begin{proof}
The result follows from Lemmas \ref{lm:bias_err_sq_jm} and \ref{lm:Jm_lip} with $m=1$.
\end{proof}
\begin{lemma}
	\label{lm:E_gk_Jmu}
	\begin{align*}
		&\E\left[\left\lVert g_k^s - \nabla J_{\mu}(\theta_k^s)\right\rVert^2\right]
		\leq \frac{224e^2d^2L^2}{n}+6L^2\E\left[\left\lVert \theta_k^s -\tilde{\theta}^s \right\rVert^2 \right ].
  	\end{align*}
\end{lemma}
\begin{proof}
	Notice that
	\begin{align*}
		&\E\left[\left\lVert g_k^s- \nabla J_\mu(\theta_k^s)\right\rVert^2\right] \\
		&=\E\left[\left\lVert \widehat{\nabla}_{n,\mu} \hat{J}^j(\theta_k^s) - \widehat{\nabla}_{n,\mu} \hat{J}^j(\tilde{\theta}^s)
		+ \widehat{\nabla}_{n,\mu}\hat{J}_{m}(\tilde{\theta}^s) - \nabla J_{\mu}(\theta_k^s) \right\rVert^2\right]\\
  		&=\E\left[\left\lVert \widehat{\nabla}_{n,\mu} \hat{J}^j(\theta_k^s)-\widehat{\nabla}_{n,\mu} \hat{J}^j(\tilde{\theta}^s)
		+ \widehat{\nabla}_{n,\mu}\hat{J}_{m}(\tilde{\theta}^s) -\widehat{\nabla}_{n,\mu}\hat{J}_{m}(\theta_k^s)\right.\right.\nonumber\\
        &\qquad \left.\left.+\widehat{\nabla}_{n,\mu}\hat{J}_{m}(\theta_k^s) - \nabla J_{\mu}(\theta_k^s) \right\rVert^2\right]\\
        &\stackrel{(a)}{=}\E\left[\left\lVert \widehat{\nabla}_{n,\mu} \hat{J}^j(\theta_k^s)-\widehat{\nabla}_{n,\mu} \hat{J}^j(\tilde{\theta}^s)
		-\mathop{\mathbb{E}}_{j\in[1,m]}\left[\widehat{\nabla}_{n,\mu}\hat{J}^{j}(\theta_k^s)-\widehat{\nabla}_{n,\mu}\hat{J}^{j}(\tilde{\theta}^s)\right]\right.\right.\nonumber\\
        &\qquad \left.\left.+\widehat{\nabla}_{n,\mu}\hat{J}_{m}(\theta_k^s) - \nabla J_{\mu}(\theta_k^s) \right\rVert^2\right]\\
        &\stackrel{(b)}{\leq}2\E\left[\left\lVert \widehat{\nabla}_{n,\mu} \hat{J}^j(\theta_k^s)-\widehat{\nabla}_{n,\mu} \hat{J}^j(\tilde{\theta}^s)
		-\mathop{\mathbb{E}}_{j\in[1,m]}\left[  \widehat{\nabla}_{n,\mu}\hat{J}^{j}(\theta_k^s)-\widehat{\nabla}_{n,\mu}\hat{J}^{j}(\tilde{\theta}^s)\right]\right\rVert^2\right]\nonumber\\
        &\quad+2\E\left[\left\lVert\widehat{\nabla}_{n,\mu}\hat{J}_{m}(\theta_k^s) - \nabla J_{\mu}(\theta_k^s) \right\rVert^2\right]\\
        &\stackrel{(c)}{\leq}2\E\left[\left\lVert \widehat{\nabla}_{n,\mu} \hat{J}^j(\theta_k^s)-\widehat{\nabla}_{n,\mu} \hat{J}^j(\tilde{\theta}^s)
		\right\rVert^2\right]
        +2\E\left[\left\lVert\widehat{\nabla}_{n,\mu}\hat{J}_{m}(\theta_k^s) - \nabla J_{\mu}(\theta_k^s) \right\rVert^2\right]\\
        &=2\E\left[\left\lVert \widehat{\nabla}_{n,\mu} \hat{J}^j(\theta_k^s)-\widehat{\nabla}_{n,\mu} \hat{J}^j(\tilde{\theta}^s) +\nabla J_{\mu}(\theta_k^s)- \nabla J_{\mu}(\theta_k^s)+\nabla J_{\mu}(\tilde{\theta}^s)-\nabla J_{\mu}(\tilde{\theta}^s)
		\right\rVert^2\right]\nonumber\\
        &\quad+2\E\left[\left\lVert\widehat{\nabla}_{n,\mu}\hat{J}_{m}(\theta_k^s) - \nabla J_{\mu}(\theta_k^s) \right\rVert^2\right]\\
    	&\stackrel{(d)}{\leq}6\E\left[\left\lVert \widehat{\nabla}_{n,\mu} \hat{J}^j(\theta_k^s) - \nabla J_{\mu}(\theta_k^s)\right\rVert^2 \right ]+ 6\E\left[\left\lVert \nabla J_{\mu}(\tilde{\theta}^s)- \widehat{\nabla}_{n,\mu} \hat{J}^j(\tilde{\theta}^s)\right\rVert^2 \right ]\nonumber\\
		&\qquad +6 \E\left[\left\lVert \nabla J_{\mu}(\theta_k^s) -\nabla J_{\mu}(\tilde{\theta}^s) \right\rVert^2 \right ]+ 2\E\left[\left\lVert\widehat{\nabla}_{n,\mu}\hat{J}_{m}(\theta_k^s) - \nabla J_{\mu}(\theta_k^s) \right\rVert^2\right] \\
      	&\stackrel{(e)}{\leq} \frac{224e^2d^2L^2}{n}+6L^2\E\left[\left\lVert \theta_k^s -\tilde{\theta}^s \right\rVert^2 \right ].
	\end{align*}
 In the above, the step $(a)$ follows from Lemma \ref{lm:Jm_EJhat}. The step $(b)$ and $(d)$ follow since $\lVert\sum_{i=1}^n a_i \rVert^2\leq n \sum_{i=1}^n\lVert\ a_i \rVert^2$.
 The step $(c)$ follows since $\E[\lVert X-E[X]\rVert^2]\leq \E[\lVert X \rVert^2]$, and the step $(e)$ follows from Lemmas
 \ref{lm:bias_err_sq_jm}, \ref{lm:bias_err_sq_jj} and \ref{lm:nabla_Jmu_lip}.
\end{proof}
\begin{lemma}
	\label{lm:E_gk_J}
	\begin{align*}
		&\E\left[\left\lVert g_k^s - \nabla J(\theta_k^s)\right\rVert^2\right]
		\leq \frac{192e^2d^2L^2}{n}+\mu^2 d^2L^2.
  	\end{align*}
\end{lemma}
\begin{proof}
	Notice that
	\begin{align*}
		&\E\left[\left\lVert g_k^s- \nabla J(\theta_k^s)\right\rVert^2\right] \\
		&=\E\left[\left\lVert \widehat{\nabla}_{n,\mu} \hat{J}^j(\theta_k^s) - \widehat{\nabla}_{n,\mu} \hat{J}^j(\tilde{\theta}^s)
		+ \widehat{\nabla}_{n,\mu}\hat{J}_{m}(\tilde{\theta}^s) - \nabla J(\theta_k^s) \right\rVert^2\right]\\
  		&=\E\left[\left\lVert \widehat{\nabla}_{n,\mu} \hat{J}^j(\theta_k^s) - \nabla J_{\mu}(\theta_k^s)+ \nabla J_{\mu}(\tilde{\theta}^s)- \widehat{\nabla}_{n,\mu} \hat{J}^j(\tilde{\theta}^s)\right.\right.\nonumber\\
		&\qquad \left.\left.+ \widehat{\nabla}_{n,\mu}\hat{J}_{m}(\tilde{\theta}^s) -\nabla J_{\mu}(\tilde{\theta}^s) + \nabla J_{\mu}(\theta_k^s) - \nabla J(\theta_k^s) \right\rVert^2\right]\\
    	&\stackrel{(a)}{\leq}4\E\left[\left\lVert \widehat{\nabla}_{n,\mu} \hat{J}^j(\theta_k^s) - \nabla J_{\mu}(\theta_k^s)\right\rVert^2 \right ]+ 4\E\left[\left\lVert \nabla J_{\mu}(\tilde{\theta}^s)- \widehat{\nabla}_{n,\mu} \hat{J}^j(\tilde{\theta}^s)\right\rVert^2 \right ]\nonumber\\
		&\qquad +4 \E\left[\left\lVert \widehat{\nabla}_{n,\mu}\hat{J}_{m}(\tilde{\theta}^s) -\nabla J_{\mu}(\tilde{\theta}^s) \right\rVert^2 \right ]+ 4\E\left[\left\lVert \nabla J_{\mu}(\theta_k^s) - \nabla J(\theta_k^s) \right\rVert^2\right]\\
      	&\stackrel{(b)}{\leq} \frac{192e^2d^2L^2}{n}+\mu^2 d^2L^2.
	\end{align*}
 	In the above, the step \((a)\) follows since $\lVert\sum_{i=1}^n a_i \rVert^2\leq n \sum_{i=1}^n\lVert\ a_i \rVert^2$. The step $(b)$ follows from Lemmas \ref{lm:bias_err_sq_jm}, \ref{lm:bias_err_sq_jj} and \ref{lm:bias}.
\end{proof}
\subsubsection*{Proof of Theorem \ref{tm:non_asym_svrg}}
To prove the theorem, we utilize techniques from \cite{reddi16,liu}.

Using the fundamental theorem of calculus, we obtain
	\begin{align}
		& J_{\mu}(\theta_k^s) - J_{\mu}(\theta_{k+1}^s) \nonumber\\
		&=\langle \nabla J_{\mu}(\theta_k^s), \theta_k^s - \theta_{k+1}^s \rangle + \int_0^1 \left \langle \nabla J_{\mu}(\theta_{k+1}^s+\tau(\theta_k^s - \theta_{k+1}^s))-\nabla J_{\mu}(\theta_k^s), \theta_k^s - \theta_{k+1}^s \right \rangle d\tau \nonumber\\
		&\stackrel{(a)}{\leq}\langle \nabla J_{\mu}(\theta_k^s), \theta_k^s - \theta_{k+1}^s \rangle + \int_0^1 \left\lVert \nabla J_{\mu}(\theta_{k+1}^s+\tau(\theta_k^s - \theta_{k+1}^s))-\nabla J_{\mu}(\theta_k^s)\right\rVert \left\lVert \theta_k^s - \theta_{k+1}^s \right \rVert d\tau \nonumber\\
		&\stackrel{(b)}{\leq} \left \langle \nabla J_{\mu}(\theta_k^s), \theta_k^s - \theta_{k+1}^s \right \rangle + L\left\lVert \theta_k^s - \theta_{k+1}^s \right\rVert^2  \int_0^1 (1-\tau) d\tau \nonumber\\
		&= \left \langle \nabla J_{\mu}(\theta_k^s), \theta_k^s - \theta_{k+1}^s \right \rangle + \frac{L}{2}\left\lVert \theta_k^s - \theta_{k+1}^s \right\rVert^2 \nonumber\\
		&= \alpha \left \langle \nabla J_{\mu}(\theta_k^s), -\mathcal{P}_\Theta(\theta_k^s,g_k^s, \alpha) \right \rangle + \frac{L\alpha^2}{2}\left\lVert \mathcal{P}_\Theta(\theta_k^s, g_k^s, \alpha) \right\rVert^2 \nonumber\\
  		&= \alpha \left \langle g_k^s -\nabla J_{\mu}(\theta_k^s), \mathcal{P}_\Theta(\theta_k^s,g_k^s, \alpha) \right \rangle - \alpha \left \langle g_k^s, \mathcal{P}_\Theta(\theta_k^s,g_k^s, \alpha) \right \rangle
        +\frac{L\alpha^2}{2}\left\lVert \mathcal{P}_\Theta(\theta_k^s, g_k^s, \alpha) \right\rVert^2 \nonumber\\
        &\stackrel{(c)}\leq \alpha \left \langle g_k^s -\nabla J_{\mu}(\theta_k^s), \mathcal{P}_\Theta(\theta_k^s,g_k^s, \alpha) \right \rangle
        +\left(\frac{L\alpha^2}{2}-\alpha\right)\left\lVert \mathcal{P}_\Theta(\theta_k^s, g_k^s, \alpha) \right\rVert^2 \nonumber\\
        &= \alpha \left \langle g_k^s -\nabla J_{\mu}(\theta_k^s), \mathcal{P}_\Theta(\theta_k^s,g_k^s, \alpha)- \mathcal{P}_\Theta(\theta_k^s,\nabla J_{\mu}(\theta_k^s), \alpha)\right \rangle \nonumber\\
        &\quad+\alpha \left \langle g_k^s -\nabla J_{\mu}(\theta_k^s),  \mathcal{P}_\Theta(\theta_k^s,\nabla J_{\mu}(\theta_k^s), \alpha)\right \rangle
        +\left(\frac{L\alpha^2}{2}-\alpha\right)\left\lVert \mathcal{P}_\Theta(\theta_k^s, g_k^s, \alpha) \right\rVert^2 \nonumber\\
        &\stackrel{(d)}{\leq} \alpha \left \lVert g_k^s -\nabla J_{\mu}(\theta_k^s)\right\rVert \left \lVert\mathcal{P}_\Theta(\theta_k^s,g_k^s, \alpha)- \mathcal{P}_\Theta(\theta_k^s,\nabla J_{\mu}(\theta_k^s), \alpha)\right \rVert \nonumber\\
        &\quad+\alpha \left \langle g_k^s -\nabla J_{\mu}(\theta_k^s),  \mathcal{P}_\Theta(\theta_k^s,\nabla J_{\mu}(\theta_k^s), \alpha)\right \rangle
        +\left(\frac{L\alpha^2}{2}-\alpha\right)\left\lVert \mathcal{P}_\Theta(\theta_k^s, g_k^s, \alpha) \right\rVert^2 \nonumber\\
        &\stackrel{(e)}{\leq} \alpha \left \lVert g_k^s -\nabla J_{\mu}(\theta_k^s)\right\rVert^2 +\left(\frac{L\alpha^2}{2}-\alpha\right)\left\lVert \mathcal{P}_\Theta(\theta_k^s, g_k^s, \alpha) \right\rVert^2  \nonumber\\
        &\quad+\alpha \left \langle g_k^s -\nabla J_{\mu}(\theta_k^s),  \mathcal{P}_\Theta(\theta_k^s,\nabla J_{\mu}(\theta_k^s), \alpha)\right \rangle
      \label{eq:1_svrg}
	\end{align}
In the above steps $(a)$ and $(d)$ follow from the fact that $\langle a,b \rangle \leq \lVert a \rVert\lVert b \rVert$. The step $(b)$ follows from Lemma \ref{lm:nabla_Jmu_lip}. The steps $(c)$ and $(e)$ follow from Lemma \ref{lm:proxy}.

	Taking expectations on both sides of \eqref{eq:1_svrg}, we obtain
	\begin{align}
		& \E\left[J_{\mu}(\theta_{k+1}^s)\right] \nonumber\\
		&\geq \E\left[J_{\mu}(\theta_k^s)\right] + \left(\alpha-\frac{L\alpha^2}{2}\right)\E\left[\left\lVert \mathcal{P}_\Theta(\theta_k^s, g_k^s, \alpha) \right\rVert^2 \right] - \alpha \E\left[\left \lVert g_k^s -\nabla J_{\mu}(\theta_k^s)\right\rVert^2 \right ]\nonumber\\
        &\quad-\alpha\E\left[ \left \langle g_k^s -\nabla J_{\mu}(\theta_k^s),  \mathcal{P}_\Theta(\theta_k^s,\nabla J_{\mu}(\theta_k^s), \alpha)\right \rangle\right]\nonumber\\
        &\stackrel{(a)}{=} \E\left[J_{\mu}(\theta_k^s)\right] + \left(\alpha-\frac{L\alpha^2}{2}\right)\E\left[\left\lVert \mathcal{P}_\Theta(\theta_k^s, g_k^s, \alpha) \right\rVert^2 \right] - \alpha \E\left[\left \lVert g_k^s -\nabla J_{\mu}(\theta_k^s)\right\rVert^2 \right ]\nonumber\\
        &\stackrel{(b)}{\geq} \E\left[J_{\mu}(\theta_k^s)\right] + \left(\alpha-\frac{L\alpha^2}{2}\right)\E\left[\left\lVert \mathcal{P}_\Theta(\theta_k^s, g_k^s, \alpha) \right\rVert^2 \right] -6L^2\alpha\E\left[\left\lVert \theta_k^s -\tilde{\theta}^s \right\rVert^2 \right ] \nonumber\\
        &\qquad -\frac{224e^2d^2L^2\alpha}{n}.
        		\label{eq:2_svrg}
	\end{align}
	In the above, the step $(a)$ follows since $\E\left[ \left \langle g_k^s -\nabla J_{\mu}(\theta_k^s),  \mathcal{P}_\Theta(\theta_k^s,\nabla J_{\mu}(\theta_k^s), \alpha)\right \rangle\right]=0$ by utilizing the fact that $\mathbb{E}\left[g_k^s\right]=\nabla J_{\mu}(\theta_k^s)$ from Lemma \ref{lm:Eg_Jmu}. The step $(b)$ follows from Lemma \ref{lm:E_gk_Jmu}.

 Now,
	\begin{align}
		&\E\left[\left\lVert \theta_{k+1}^s - \tilde{\theta}^s\right\rVert^2\right] \nonumber\\
        &= \E\left[\left\lVert \theta_{k+1}^s - \theta_{k}^s +\theta_{k}^s  -\tilde{\theta}^s\right\rVert^2\right]\nonumber\\
         &= \E\left[\left\lVert \theta_{k+1}^s - \theta_{k}^s\right\rVert^2 + 2\left\langle\theta_{k+1}^s - \theta_{k}^s, \theta_{k}^s  -\tilde{\theta}^s \right\rangle+\left\lVert\theta_{k}^s  -\tilde{\theta}^s\right\rVert^2\right]\nonumber\\
         &= \E\left[\alpha^2\left\lVert \mathcal{P}_\Theta(\theta_k^s, g_k^s, \alpha)\right\rVert^2 + 2\alpha\left\langle\mathcal{P}_\Theta(\theta_k^s, g_k^s, \alpha), \theta_{k}^s  -\tilde{\theta}^s \right\rangle+\left\lVert\theta_{k}^s  -\tilde{\theta}^s\right\rVert^2\right]\nonumber\\
        &\stackrel{(a)}{\leq} \E\left[\alpha^2\left\lVert \mathcal{P}_\Theta(\theta_k^s, g_k^s, \alpha)\right\rVert^2 + \frac{\alpha}{\beta}\left\lVert\mathcal{P}_\Theta(\theta_k^s, g_k^s, \alpha)\right\rVert^2+ \alpha \beta\left\lVert\theta_{k}^s  -\tilde{\theta}^s \right\rVert^2+\left\lVert\theta_{k}^s  -\tilde{\theta}^s\right\rVert^2\right]\nonumber\\
        &= \left(\alpha^2+\frac{\alpha}{\beta}\right) \E\left[\left\lVert \mathcal{P}_\Theta(\theta_k^s, g_k^s, \alpha)\right\rVert^2 \right]+ \left(1+\alpha \beta \right)\E\left[\left\lVert\theta_{k}^s  -\tilde{\theta}^s \right\rVert^2\right],
           \label{eq:3_svrg}
	\end{align}
	where the step $(a)$ follows from the fact that $\langle a,b\rangle \leq \frac{\lVert a\rVert^2}{2\beta}+\frac{\lVert b\rVert^2\beta}{2},\;\beta > 0$.

	Let
	\begin{align}
		R_k^s = \E\left[J_{\mu}(\theta_k^s)\right] - b_k \E \left[\left\lVert \theta_k^s - \tilde{\theta}^s\right\rVert^2\right].
		\label{eq:4_svrg}
	\end{align}
 Now,
	\begin{align}
		&R_{k+1}^s= \E\left[J_{\mu}(\theta_{k+1}^s)\right] - b_{k+1} \E \left[\left\lVert \theta_{k+1}^s - \tilde{\theta}^s\right\rVert^2\right]\nonumber\\
		&\stackrel{(a)}{\geq} \E\left[J_{\mu}(\theta_k^s)\right] + \left(\alpha-\frac{L\alpha^2}{2}\right)\E\left[\left\lVert \mathcal{P}_\Theta(\theta_k^s, g_k^s, \alpha) \right\rVert^2 \right] -6L^2\alpha\E\left[\left\lVert \theta_k^s -\tilde{\theta}^s \right\rVert^2 \right ] \nonumber\\
        &\qquad -\frac{224e^2d^2L^2\alpha}{n}- b_{k+1} \mathbb{E}\left[\left\lVert \theta_{k+1}^s - \tilde{\theta}^s\right\rVert^2\right]\nonumber\\
		&\stackrel{(b)}{\geq} \E\left[J_{\mu}(\theta_k^s)\right] + \left(\alpha-\frac{L\alpha^2}{2} - b_{k+1}\left(\alpha^2+\frac{\alpha}{\beta}\right)\right)\E\left[\left\lVert \mathcal{P}_\Theta(\theta_k^s, g_k^s, \alpha) \right\rVert^2 \right] \nonumber\\
        &\quad - \left(6L^2\alpha+b_{k+1} \left(1+\alpha \beta \right)\right)\E\left[\left\lVert \theta_k^s -\tilde{\theta}^s \right\rVert^2 \right ] -\frac{224e^2d^2L^2\alpha}{n}.
		\label{eq:5_svrg}
	\end{align}
In the above, the step \((a)\) follows from \eqref{eq:2_svrg}, and the step \((b)\) follows from \eqref{eq:3_svrg}.

	Let
	\begin{align}
		b_k &= \label{eq:6_svrg}
		\begin{cases}
			x +  b_{k+1}\left(1+y\right)&\textrm{ for }k\in\{0,l-1\}\\
			0&\textrm{ for }k\ge l,
		\end{cases}
		\shortintertext{where}
		x&=6L^2\alpha,\;
		y= \alpha \beta;\; \textrm{ (recall that $\beta>0$ is introduced in \eqref{eq:3_svrg})}\nonumber
	\end{align}

	By solving the recursion \eqref{eq:6_svrg}, we obtain
	\begin{align}
		\label{eq:rec1}
		b_k &= \frac{x}{y} \left(\left(1+y\right)^{l-k} -1\right).
	\end{align}
	We can see that
	\begin{align}
		\label{eq:rec2}
		&b_k \leq b_0 ,\;\forall k, \textrm{ and}\nonumber\\
        &b_0 \leq \frac{6L^2}{\beta} \left(\left(1+\alpha \beta\right)^l-1\right),
	\end{align}
Let
\begin{align}
\label{eq:ck}
c_k&= \alpha-\frac{L\alpha^2}{2} - b_{k+1}\left(\alpha^2+\frac{\alpha}{\beta}\right), \textrm{ and}\nonumber\\
\bar{c}&=\argmin_{k\in\{0,\cdots,l-1\}} \alpha-\frac{L\alpha^2}{2} - b_{k+1}\left(\alpha^2+\frac{\alpha}{\beta}\right).
\end{align}
From \eqref{eq:5_svrg}, \eqref{eq:6_svrg} and  \eqref{eq:ck}, we obtain
 \begin{align}
 &R_{k+1}^s \geq R_{k}^s + \bar{c} \E\left[\left\lVert \mathcal{P}_\Theta(\theta_k^s, g_k^s, \alpha) \right\rVert^2 \right]
        -\frac{224e^2d^2L^2\alpha}{n}.
        \label{eq:7_svrg}
\end{align}

	Summing \eqref{eq:7_svrg} from $k=0,\cdots,l-1$, we obtain
	\begin{align}
 &R_{l}^s \geq R_{0}^s + \sum_{k=0}^{l-1} \bar{c} \E\left[\left\lVert \mathcal{P}_\Theta(\theta_k^s, g_k^s, \alpha) \right\rVert^2 \right]
        -\frac{224e^2d^2L^2\alpha l}{n}.
		  \label{eq:8_svrg}
	\end{align}
 Now, from \eqref{eq:4_svrg} we obtain
	\begin{align}
		R_l^s &= \E\left[J_{\mu}(\theta_l^s)\right] - b_l \E\left[\left\lVert \theta_l^s - \tilde{\theta}^s\right\rVert^2\right]
		\stackrel{(a)}{=}\E\left[J_{\mu}(\theta_l^s)\right]=\E\left[J_{\mu}(\tilde{\theta}^{s+1})\right],\nonumber\\
		R_0^s &= \E\left[J_{\mu}(\theta_0^s)\right] - b_0 \E\left[\left\lVert \theta_0^s - \tilde{\theta}^s\right\rVert^2\right]
		\stackrel{(b)}{=}\E\left[J_{\mu}(\theta_0^s)\right]=\E\left[J_{\mu}(\tilde{\theta}^{s})\right],
		\label{eq:R}
	\end{align}
where \((a)\) follows since $b_l=0$ from \eqref{eq:6_svrg}, and \((b)\) follows since $\theta_0^s = \tilde{\theta}^s$.

From \eqref{eq:8_svrg} and \eqref{eq:R}, we obtain
	\begin{align}
 &\E\left[J_{\mu}(\tilde{\theta}^{s+1})\right] \geq \E\left[J_{\mu}(\tilde{\theta}^{s}) \right]+ \sum_{k=0}^{l-1} \bar{c} \E\left[\left\lVert \mathcal{P}_\Theta(\theta_k^s, g_k^s, \alpha) \right\rVert^2 \right]
        -\frac{224e^2d^2L^2\alpha l}{n}.
		  \label{eq:8_1_svrg}
	\end{align}

	Summing \eqref{eq:8_1_svrg} from $s=0,\cdots,S-1$, we obtain
	\begin{align}
	   &\E\left[J_{\mu}(\tilde{\theta}^{S})\right] \geq \E\left[J_{\mu}(\tilde{\theta}^{0})\right]
  + \sum_{s=0}^{S-1}\sum_{k=0}^{l-1} \bar{c} \E\left[\left\lVert \mathcal{P}_\Theta(\theta_k^s, g_k^s, \alpha) \right\rVert^2 \right]
        -\frac{224e^2d^2L^2\alpha Sl}{n}.
		\label{eq:9_svrg}
	\end{align}
Re-arranging \eqref{eq:9_svrg}, we obtain
	\begin{align}
	   \sum_{s=0}^{S-1}\sum_{k=0}^{l-1} \bar{c} \E\left[\left\lVert \mathcal{P}_\Theta(\theta_k^s, g_k^s, \alpha) \right\rVert^2 \right]
    &\leq
    \E\left[J_{\mu}(\tilde{\theta}^{S})\right] - \E\left[J_{\mu}(\tilde{\theta}^{0})\right] +\frac{224e^2d^2L^2\alpha Sl}{n}\nonumber\\
    &\leq
    J_{\mu}^*-J_{\mu}(\theta^{0}_{0}) +\frac{224e^2d^2L^2\alpha Sl}{n}\nonumber\\
    &\stackrel{(a)}{\leq}
    J^*-J(\theta^{0}_{0})+L\mu^2 +\frac{224e^2d^2L^2\alpha Sl}{n}.
		\label{eq:10_svrg}
	\end{align}
In the above, $J_{\mu}^*=\max_{\theta\in\Theta}J_{\mu}(\theta)$ and $J^*=\max_{\theta\in\Theta}J(\theta)$. The step $(a)$ follows from Lemma \ref{lm:Jmu_J} and from the fact that $\lvert J_{\mu}^* - J^* \rvert=\lvert \max_{\theta\in\Theta}J_{\mu}(\theta)- \max_{\theta\in\Theta}J(\theta)\rvert \leq \max_{\theta\in\Theta}\lvert J_{\mu}(\theta)- J(\theta)\rvert$.

	From the theorem statement, we have $\alpha=\frac{1}{6 Ld}$ and $l=d$. Let $\beta=6 L$.
   	From \eqref{eq:rec2}, we have
	\begin{align}
		b_0&\leq \frac{6L^2}{\beta} \left(\left(1+\alpha\beta\right)^l-1\right)\nonumber\\
           &\leq L \left(\left(1+\frac{1}{d}\right)^d-1\right)\nonumber\\
           &\stackrel{(a)}{\leq} L \left(e-1\right)\leq 2L.
           \label{eq:b0}
	\end{align}
In the above, the step $(a)$ follows from the fact that $\lim_{a\to\infty}(1+1/a)^a=e,\; a>0$.

 Using \eqref{eq:b0} and \eqref{eq:rec2} on \eqref{eq:ck}, we obtain
\begin{align}
\label{eq:c_tilde}
\bar{c}&=\argmin_{k\in\{0,\cdots,l-1\}} \alpha-\frac{L\alpha^2}{2} - b_{k+1}\left(\alpha^2+\frac{\alpha}{\beta}\right)\nonumber\\
&\geq \alpha-\frac{L\alpha^2}{2} - 2L\left(\alpha^2+\frac{\alpha}{6 L}\right)\nonumber\\
&=\frac{1}{6 Ld}-\frac{1}{72 Ld^2} -\frac{1}{18 Ld^2}-\frac{1}{18 Ld }\nonumber\\
&=\tilde{c}.
\end{align}
We can see that $\tilde{c}>0$.

Now, applying \eqref{eq:c_tilde} in \eqref{eq:10_svrg}, we obtain
	\begin{align}
	   \sum_{s=0}^{S-1}\sum_{k=0}^{l-1} \tilde{c} \E\left[\left\lVert \mathcal{P}_\Theta(\theta_k^s, g_k^s, \alpha) \right\rVert^2 \right]
    &\leq J^*-J(\theta^{0}_{0})+L\mu^2 +\frac{112e^2dL Sl}{3 n}.
		\label{eq:11_svrg}
	\end{align}

	Since $\mathbb{P}(Q=s, R=k)=\frac{1 }{Sl}$, we obtain
	\begin{align}
		\E\left[\left \lVert \mathcal{P}_\Theta(\theta_R^Q, g_R^Q), \alpha) \right \rVert^2\right]
		&= \frac{\sum\limits_{s=0}^{S-1}\sum\limits_{k=0}^{l-1}\tilde{c}\mathbb{E}\left[\left\lVert \mathcal{P}_\Theta(\theta_k^s,g_k^s), \alpha) \right\rVert^2\right]}{\sum_{s=0}^{S-1}\sum_{k=0}^{l-1}\tilde{c}}\nonumber\\
		&\leq \frac{J^*-J(\theta^{0}_{0})+L\mu^2 +\frac{112e^2dL Sl}{3 n}}{Sl\tilde{c}}.
		\label{eq:12_svrg}
	\end{align}
Now,
\begin{align}
&\E\left[\left \lVert \mathcal{P}_\Theta(\theta_R^Q,\nabla J(\theta_R^Q), \alpha) \right \rVert^2\right]\nonumber\\
&=\E\left[\left \lVert \mathcal{P}_\Theta(\theta_R^Q,\nabla J(\theta_R^Q), \alpha)
- \mathcal{P}_\Theta(\theta_R^Q, g_R^Q, \alpha)+\mathcal{P}_\Theta(\theta_R^Q,g_R^Q, \alpha)\right \rVert^2\right]\nonumber\\
&\stackrel{(a)}{\leq}2\E\left[\left \lVert \mathcal{P}_\Theta(\theta_R^Q,\nabla J(\theta_R^Q), \alpha)- \mathcal{P}_\Theta(\theta_R^Q, g_R^Q, \alpha)\right \rVert^2\right] +2\E\left[\left \lVert\mathcal{P}_\Theta(\theta_R^Q,g_R^Q, \alpha)\right \rVert^2\right]\nonumber\\
&\stackrel{(b)}{\leq}2\E\left[\left \lVert \nabla J(\theta_R^Q)  -  g_R^Q\right \rVert^2\right]+2\E\left[\left \lVert\mathcal{P}_\Theta(\theta_R^Q,g_R^Q, \alpha)\right \rVert^2\right]\nonumber\\
&\stackrel{(c)}{\leq}\frac{384e^2d^2L^2}{n}+2\mu^2 d^2L^2 +2\E\left[\left \lVert\mathcal{P}_\Theta(\theta_R^Q,g_R^Q, \alpha)\right \rVert^2\right]\nonumber\\
&\stackrel{(d)}{\leq}\frac{384e^2d^2L^2}{n}+2\mu^2 d^2L^2 +2 \frac{J^*-J(\theta^{0}_{0})+L\mu^2 +\frac{112e^2dL Sl}{3 n}}{Sl\tilde{c}}\nonumber\\
&\stackrel{(e)}{=}\frac{384e^2d^2L^2}{Sl}+\frac{2 d^2L^2}{Sl} +2 \frac{J^*-J(\theta^{0}_{0})+\frac{L}{Sl} +\frac{112e^2dL}{3}}{Sl\tilde{c}}\nonumber\\
&=\frac{2\left(J^*-J(\theta^{0}_{0})\right)}{Sl\tilde{c}}+\frac{384e^2d^2L^2}{Sl}+\frac{2 d^2L^2}{Sl} +\frac{2L}{S^2l^2\tilde{c}} +\frac{224e^2dL}{3 Sl\tilde{c}}.
\end{align}
In the above step $(a)$ follows from the fact that $\lVert\sum_{i=1}^n a_i \rVert^2 \leq n\sum_{i=1}^n \lVert a_i \rVert^2 $, and step $(b)$ follows from Lemma \ref{lm:proxy}. The step $(c)$ follows from Lemma \ref{lm:E_gk_J} since $\mathbb{P}(Q=s, R=k)=\nicefrac{1 }{Sl}$. The step $(d)$ follows from \eqref{eq:12_svrg}. The step $(e)$ follows since $n=Sl$ and $\mu=\frac{1}{\sqrt{Sl}}$
\hfill{\qed}
}

\subsection{Analysis of OffP-REINFORCE}
\label{sec:proofs-pg}
\subsection*{Proof of Theorem \ref{tm:pg}}
The policy gradient estimate $\widehat{\nabla} J_m(\theta)$ in \eqref{eq:hat_nabla_j_theta} is an unbiased estimate of $\nabla J(\theta)$, where
\begin{align}
    \label{eq:nabla_j_theta}
    \nabla J(\theta)
    &= \mathbb{E}_{\pi_\theta}\left[\left(\sum_{t=0}^{T-1}\nabla\log\pi_\theta(A_t\mid S_t)\right)\left(\sum_{t=0}^{T-1}\gamma^t R_{t+1}\right)\right]\nonumber\\
    &= \mathbb{E}_b\left[\left(\prod_{t=0}^{T-1}\frac{\pi_\theta(A_t\mid S_t)}{b(A_t\mid S_t)}\right)
    \left(\sum_{t=0}^{T-1}\nabla\log\pi_\theta(A_t\mid S_t)\right)\left(\sum_{t=0}^{T-1}\gamma^t R_{t+1}\right)\right]\nonumber\\
    &= \mathbb{E}_b\left[\sum_{t=0}^{T-1}\nabla\log\pi_\theta(A_t\mid S_t)
     \left(\prod_{i=0}^{t}\frac{\pi_\theta(A_i\mid S_i)}{b(A_i\mid S_i)}\right)\left(\sum_{i=t}^{T-1}\gamma^i R_{i+1}\right)\right],
\end{align}
and
    \begin{align}
    \label{eq:unbias_hat_nabla_j_theta}
    &\mathbb{E}_b\left[\widehat{\nabla} J_m(\theta)\right] \nonumber\\
    &= \mathbb{E}_b\left[\frac{1}{m}\sum_{j=1}^{m}\left[\sum_{t=0}^{T^j-1}\nabla\log\pi_\theta(A_t^j\mid S_t^j)
    \left(\prod_{i=0}^{t}\frac{\pi_\theta(A_i^j\mid S_i^j)}{b(A_i^j\mid S_i^j)}\right)\left(\sum_{i=t}^{T^j-1}\gamma^i R_{i+1}^j\right)\right]\right]\nonumber\\
    &= \mathbb{E}_b\left[\sum_{t=0}^{T-1}\nabla\log\pi_\theta(A_t\mid S_t)\left(\prod_{i=0}^{t}\frac{\pi_\theta(A_i\mid S_i)}{b(A_i\mid S_i)}\right)
    \left(\sum_{i=t}^{T-1}\gamma^i R_{i+1}\right)\right] \nonumber\\
    &= \nabla J(\theta).
\end{align}

{
Now, from \eqref{eq:Jm} and \eqref{eq:hat_nabla_j_theta}, we obtain
 \begin{align}
     \label{eq:nabla_Jm_J}
     \nabla \hat{J}_{m}(\theta)
     &=\nabla\left[\frac{1}{m}\sum_{j=1}^{m} \sum_{t=0}^{T^j-1}\gamma^t R_{t+1}^j \left(\prod_{i=0}^{t}\frac{\pi_{\theta}(A_i^j \mid S_i^j)}{b(A_i^j \mid S_i^j)} \right) \right] \nonumber\\
      &=\frac{1}{m}\sum_{j=1}^{m} \sum_{t=0}^{T^j-1}\gamma^t R_{t+1}^j \nabla\left(\prod_{i=0}^{t}\frac{\pi_{\theta}(A_i^j \mid S_i^j)}{b(A_i^j \mid S_i^j)} \right)  \nonumber\\
     &\stackrel{(a)}{=}\frac{1}{m}\sum_{j=1}^{m}\sum_{t=0}^{T^j-1}\gamma^t R_{t+1}^j \left(\prod_{i=0}^{t}\frac{\pi_{\theta}(A_i^j \mid S_i^j)}{b(A_i^j \mid S_i^j)} \right)
     \left(\sum_{i=0}^{t}\nabla\log\pi_\theta(A_i^j\mid S_i^j)\right)\nonumber\\
     &=\frac{1}{m}\sum_{j=1}^{m}\left[\sum_{t=0}^{T^j-1}\nabla\log\pi_\theta(A_t^j\mid S_t^j)
     \left(\prod_{i=0}^{t}\frac{\pi_\theta(A_i^j\mid S_i^j)}{b(A_i^j\mid S_i^j)}\right)\left(\sum_{i=t}^{T^j-1}\gamma^i R_{i+1}^j\right)\right]\nonumber\\
     &=\widehat{\nabla} J_m(\theta),
 \end{align}
 where \((a)\) follows from the fact that $\nabla f(x)=f(x)\nabla\log(f(x))$.
 }

We can see that $\nabla J(\theta)$ is $L$-Lipschitz w.r.t $\theta$ and $\left\lVert \widehat{\nabla} J_m(\theta) \right\rVert \leq L$ a.s. using \eqref{eq:unbias_hat_nabla_j_theta}, \eqref{eq:nabla_Jm_J} and Lemmas \ref{lm:Jm_lip} and \ref{lm:J_lip}.

{
   We have $\forall 0<m<\infty$, $ \mathbb{E}_b\left[\widehat{\nabla} J_m(\theta)\right]
     = \nabla J(\theta)$, and $\left\lVert \widehat{\nabla} J_m(\theta) \right\rVert \leq L$ a.s.
     From Lemma \ref{lm:hayes}, we obtain
\begin{align}
    \label{eq:nabla_G_prob}
    &\mathbb{P}\left(\left\lVert \widehat{\nabla} J_m(\theta) - \nabla J(\theta) \right\rVert > \epsilon\right) \leq 2e^2\exp\left(\frac{-m\epsilon^2}{8L^2}\right), 
\end{align}
and
\begin{align}
    \mathbb{E}_b\left[\left\lVert \widehat{\nabla} J_m(\theta)- \nabla J(\theta) \right\rVert \right]
    &= \int_{0}^{\infty} \mathbb{P}\left( \left\lVert \widehat{\nabla} J_m(\theta)- \nabla J(\theta) \right\rVert > \epsilon\right)d\epsilon \nonumber\\
    &\stackrel{(a)}{\leq} \int_{0}^{\infty} 2e^2\exp\left(-\frac{m{\epsilon}^2}{8L^2}\right) d\epsilon\nonumber\\
    &\stackrel{(b)}{=} \frac{2\sqrt{2\pi}e^2L}{\sqrt{m}}.\label{eq:nabla_J_err}
\end{align}
In the above, the step $(a)$ follows from \eqref{eq:nabla_G_prob}, and the step $(b)$ follows from the fact that\\ $\int_{0}^{\infty}exp(-a\epsilon^2)d\epsilon=\frac{1}{2}\sqrt{\frac{\pi}{a}},\;\forall a>0$.
}

By using a completely parallel argument to the initial passage in the proof of Theorem \ref{tm:non_asym} leading up to \eqref{eq:1}, we obtain
\begin{align}
    & J(\theta_k) - J(\theta_{k+1}) \nonumber\\
    &\leq L\alpha \left\lVert \nabla J (\theta_k) - \widehat{\nabla} J_m(\theta_k)\right \rVert + \frac{L\alpha^2}{2}\left\lVert \widehat{\nabla} J_m(\theta_k)\right\rVert^2
    -\alpha \left \lVert \mathcal{P}_\Theta(\theta_k,\nabla J(\theta_k), \alpha) \right \rVert^2. \label{eq:p1}
\end{align}
Summing \eqref{eq:p1} from $k=0,\cdots,N-1$, we obtain
\begin{align}
    &\sum_{k=0}^{N-1}\alpha \left \lVert \mathcal{P}_\Theta(\theta_k,\nabla J(\theta_k), \alpha) \right \rVert^2 \nonumber\\
    &\leq \left(J^* - J(\theta_{0}) \right)+ L\alpha \sum_{k=0}^{N-1}\left\lVert \nabla J (\theta_k) - \widehat{\nabla} J_m(\theta_k)\right \rVert
    +\frac{L\alpha^2}{2}\sum_{k=0}^{N-1} \left\lVert \widehat{\nabla} J_m(\theta_k)\right\rVert^2.\label{eq:p2}
\end{align}
Taking expectations on both sides of \eqref{eq:p2}, we obtain
\begin{align}
    &\sum_{k=0}^{N-1}\alpha \mathbb{E}\left[\left \lVert \mathcal{P}_\Theta(\theta_k,\nabla J(\theta_k), \alpha) \right \rVert^2 \right]\nonumber\\
    &\leq \left(J^* - J(\theta_{0}) \right)+ \frac{L\alpha^2}{2}\sum_{k=0}^{N-1}\mathbb{E}\left[\left\lVert \widehat{\nabla} J_m(\theta_k)\right\rVert^2\right]
    +L\alpha\sum_{k=0}^{N-1} \mathbb{E}\left[\left\lVert \nabla J(\theta_k) - \widehat{\nabla} J_m(\theta_k)\right \rVert\right] \nonumber\\
    &{\leq \left(J^* - J(\theta_{0}) \right)+\frac{L^3}{2}\alpha^2 N+ 2\sqrt{2\pi}e^2L^2\frac{\alpha N}{\sqrt{m}},\label{eq:p3}}
\end{align}
where the last inequality follows from \eqref{eq:nabla_J_err}, and from the fact that $\left\lVert \widehat{\nabla} J_m(\theta) \right\rVert \leq L$ a.s.

Since $\mathbb{P}(R=k)=\frac{1}{N}$, we obtain
\begin{align*}
    \mathbb{E}\left[\left\lVert \mathcal{P}_\Theta(\theta_R, \nabla J (\theta_R), \alpha) \right\rVert^2\right]
    &= \frac{\sum\limits_{k=0}^{N-1}\alpha\mathbb{E}\left[\left\lVert \mathcal{P}_\Theta(\theta_k,\nabla J(\theta_k), \alpha) \right\rVert^2\right]}{\sum_{k=0}^{N-1}\alpha}\\
    &\leq \frac{\left(J^* - J(\theta_{0}) \right)+\frac{L^3}{2} \alpha^2 N + 2\sqrt{2\pi}e^2L^2\frac{\alpha N}{\sqrt{m}}}{\sum_{k=0}^{N-1}\alpha}.
\end{align*}
Since $\alpha=\frac{1}{\sqrt{N}}$ {and $m=N$}, we obtain
\begin{align*}
    &\mathbb{E}\left[\left\lVert \mathcal{P}_\Theta(\theta_R, \nabla J (\theta_R), \alpha) \right\rVert^2\right]
    {\leq \frac{\left(J^* - J(\theta_{0}) \right)}{\sqrt{N}}+\frac{\frac{L^3}{2}+2\sqrt{2\pi}e^2L^2}{\sqrt{N}}.}
\end{align*}
\hfill{\qed}
\section{Simulation analysis}
\label{sec:exps}
We conducted experiments on an control problem called CartPole from OpenAI Gym toolkit \cite{openai}. The problem is to balance a pole which is attached to a moving cart. The state space is continuous and each state is a quadruple (cart position, cart velocity, pole angle, pole velocity at tip) and the action space is discrete (push cart to the left and push cart to the right). We fixed the initial state. The problem is reset to the initial state either after $200$ steps, the pole tilt more than $15$ degrees from vertical, or the cart moves more than $2.4$ units from the centre. We receive a reward of $+1$ for each timestep in which the pole is upright.

We have used the samples collected using an $\epsilon$-greedy behavior policy and a target policy which follows an exponential softmax distribution. We have compared the performance of OffP-SF, OffP-SF-SVRG and OffP-REINFORCE algorithms.
In Figure \ref{fg:simulation}, we plot the performance of the aforementioned algorithms.
\begin{figure}
    \begin{subfigure}[b]{0.49\textwidth}
            \caption{Training}
                        \begin{center}
                            \centerline{\includegraphics[width=0.9\columnwidth]{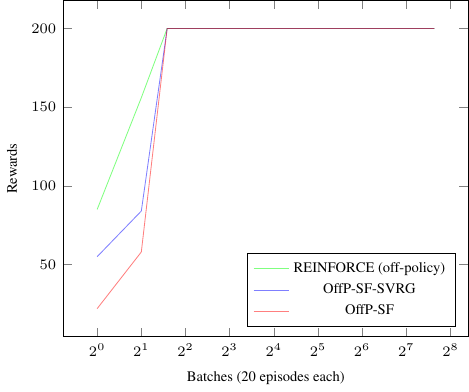}}
                        \end{center}
            \end{subfigure}
        \hfil
        \begin{subfigure}[b]{0.49\textwidth}
                \caption{Testing}
                           \begin{center}
                                \centerline{\includegraphics[width=0.9\columnwidth]{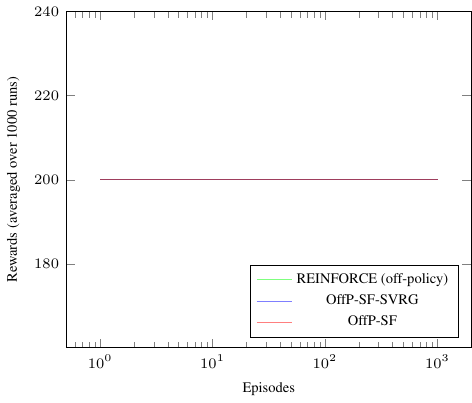}}
                            \end{center}
                \end{subfigure}
    \caption{CartPole with fixed initial state}
    \label{fg:simulation}
\end{figure}

\section{Conclusions and future work}
\label{sec:conclusions}
We proposed two policy gradient algorithms for off-policy control in a RL context. Both algorithms incorporated a smoothed functional scheme for gradient estimation.
For both algorithms, we provided non-asymptotic bounds that establish convergence to an approximate stationary point.

As future work, it would be interesting to study the global convergence properties of our algorithms under additional assumptions such as those used in \cite{zhangJ2020,zhangK2020}. An orthogonal research direction is to incorporate feature-based representations and function approximation together with smoothed functional gradient estimation, and study the non-asymptotic performance of the resulting actor-critic algorithms. Another direction of future work is to check if our algorithms are globally convergent under additional assumptions such as those in \citep{liu2020}.
\bibliography{offP_rl}

\begin{thebibliography}{46}
\providecommand{\natexlab}[1]{#1}
\providecommand{\url}[1]{\texttt{#1}}
\expandafter\ifx\csname urlstyle\endcsname\relax
  \providecommand{\doi}[1]{doi: #1}\else
  \providecommand{\doi}{doi: \begingroup \urlstyle{rm}\Url}\fi

\bibitem[Agarwal et~al.(2020)Agarwal, Kakade, Lee, and Mahajan]{agarwal2020}
A.~Agarwal, S.~M. Kakade, J.~D. Lee, and G.~Mahajan.
\newblock Optimality and approximation with policy gradient methods in markov
  decision processes.
\newblock In \emph{Conference on Learning Theory}, volume 125 of
  \emph{Proceedings of Machine Learning Research}, pages 64--66. PMLR, 09--12
  Jul 2020.

\bibitem[Allen-Zhu and Hazan(2016)]{allen2016variance}
Z.~Allen-Zhu and E.~Hazan.
\newblock Variance reduction for faster non-convex optimization.
\newblock In \emph{International conference on machine learning}, pages
  699--707. PMLR, 2016.

\bibitem[Bertsekas and Tsitsiklis(1996)]{ndp_book}
D.~P. Bertsekas and J.~N. Tsitsiklis.
\newblock \emph{Neuro-Dynamic Programming}.
\newblock Athena Scientific, 1st edition, 1996.
\newblock ISBN 1886529108.

\bibitem[Bhandari and Russo(2019)]{bhandari2019}
J.~Bhandari and D.~Russo.
\newblock Global optimality guarantees for policy gradient methods.
\newblock \emph{CoRR}, abs/1906.01786, 2019.

\bibitem[Bhandari and Russo(2020)]{bhandari2020}
J.~Bhandari and D.~Russo.
\newblock A note on the linear convergence of policy gradient methods.
\newblock \emph{CoRR}, abs/2007.11120, 2020.

\bibitem[Bhatnagar(2010)]{shalabh2010}
S.~Bhatnagar.
\newblock An actor-critic algorithm with function approximation for discounted
  cost constrained markov decision processes.
\newblock \emph{Systems \& Control Letters}, 59:\penalty0 760--766, 2010.

\bibitem[Bhatnagar and Kumar(2004)]{shalabh2004spsaRL}
S.~Bhatnagar and S.~Kumar.
\newblock {A simultaneous perturbation stochastic approximation-based
  actor-critic algorithm for Markov decision processes}.
\newblock \emph{IEEE Transactions on Automatic Control}, 49\penalty0
  (4):\penalty0 592--598, 2004.

\bibitem[Bhatnagar et~al.(2013)Bhatnagar, Prasad, and Prashanth.]{shalabh_book}
S.~Bhatnagar, H.~Prasad, and L.~A. Prashanth.
\newblock \emph{Stochastic recursive algorithms for optimization. Simultaneous
  perturbation methods}, volume 434.
\newblock Springer-Verlag London, 01 2013.

\bibitem[Brockman et~al.(2016)Brockman, Cheung, Pettersson, Schneider,
  Schulman, Tang, and Zaremba]{openai}
G.~Brockman, V.~Cheung, L.~Pettersson, J.~Schneider, J.~Schulman, J.~Tang, and
  W.~Zaremba.
\newblock Openai gym, 2016.

\bibitem[Cutkosky and Orabona(2019)]{storm}
A.~Cutkosky and F.~Orabona.
\newblock Momentum-based variance reduction in non-convex sgd.
\newblock In \emph{Advances in Neural Information Processing Systems},
  volume~32, 2019.

\bibitem[Degris et~al.(2012)Degris, White, and Sutton]{degris2012}
T.~Degris, M.~White, and R.~S. Sutton.
\newblock Off-policy actor-critic.
\newblock In \emph{International Conference on Machine Learning}, pages
  179--186, 2012.

\bibitem[Fang et~al.(1990)Fang, Kotz, and Ng]{fang_book}
K.~Fang, S.~Kotz, and K.~Ng.
\newblock \emph{Symmetric multivariate and related distributions}.
\newblock Number~36 in Monographs on statistics and applied probability.
  Chapman \& Hall, 1990.
\newblock ISBN 0412314304.

\bibitem[Fazel et~al.(2018)Fazel, Ge, Kakade, and Mesbahi]{fazel2018}
M.~Fazel, R.~Ge, S.~Kakade, and M.~Mesbahi.
\newblock Global convergence of policy gradient methods for the linear
  quadratic regulator.
\newblock In \emph{International Conference on Machine Learning}, volume~80 of
  \emph{Proceedings of Machine Learning Research}, pages 1467--1476. PMLR,
  10--15 Jul 2018.

\bibitem[Flaxman et~al.(2005)Flaxman, Kalai, and McMahan]{flaxman}
A.~D. Flaxman, A.~T. Kalai, and H.~B. McMahan.
\newblock Online convex optimization in the bandit setting: Gradient descent
  without a gradient.
\newblock In \emph{ACM-SIAM Symposium on Discrete Algorithms}, pages 385--394,
  2005.
\newblock ISBN 0898715857.

\bibitem[Fu(2006)]{Fu2006b}
M.~C. Fu.
\newblock Gradient estimation.
\newblock In S.~G. Henderson and B.~L. Nelson, editors, \emph{Handbooks in
  Operations Research and Management Science: Simulation}, chapter~19, pages
  575--616. Elsevier, 2006.

\bibitem[Fu(2015)]{Fu2015}
M.~C. Fu.
\newblock Stochastic gradient estimation.
\newblock In M.~C. Fu, editor, \emph{Handbook on Simulation Optimization},
  chapter~5. Springer, 2015.

\bibitem[Gao et~al.(2018)Gao, Jiang, and Zhang]{gao2018}
X.~Gao, B.~Jiang, and S.~Zhang.
\newblock On the information-adaptive variants of the admm: An iteration
  complexity perspective.
\newblock \emph{Journal of Scientific Computing}, 76\penalty0 (1):\penalty0
  327--363, 2018.
\newblock ISSN 1573-7691.

\bibitem[Ghadimi and Lan(2013)]{ghadimi2013}
S.~Ghadimi and G.~Lan.
\newblock Stochastic first- and zeroth-order methods for nonconvex stochastic
  programming.
\newblock \emph{SIAM J. Optim.}, 23:\penalty0 2341--2368, 2013.

\bibitem[Hayes(2005)]{hayes2005large}
T.~Hayes.
\newblock A large-deviation inequality for vector-valued martingales.
\newblock \emph{Combinatorics, Probability and Computing}, 2005.

\bibitem[Johnson and Zhang(2013)]{johnson13}
R.~Johnson and T.~Zhang.
\newblock Accelerating stochastic gradient descent using predictive variance
  reduction.
\newblock In \emph{Neural Information Processing Systems}, pages 315--323.
  Curran Associates Inc., 2013.

\bibitem[Kakade(2001)]{kakade2001natural}
S.~M. Kakade.
\newblock A natural policy gradient.
\newblock \emph{Advances in neural information processing systems}, 14, 2001.

\bibitem[Katkovnik and Kulchitsky(1972)]{katkovnik1972}
V.~Katkovnik and Y.~Kulchitsky.
\newblock Convergence of a class of random search algorithms.
\newblock \emph{Automation and Remote Control}, 33:\penalty0 1321--1326, 1972.

\bibitem[Liu et~al.(2018)Liu, Li, Chen, Haupt, and Amini]{liu}
S.~Liu, X.~Li, P.~Chen, J.~Haupt, and L.~Amini.
\newblock Zeroth-order stochastic projected gradient descent for nonconvex
  optimization.
\newblock In \emph{IEEE Global Conference on Signal and Information
  Processing}, pages 1179--1183, 2018.

\bibitem[Liu et~al.(2020)Liu, Zhang, Basar, and Yin]{liu2020}
Y.~Liu, K.~Zhang, T.~Basar, and W.~Yin.
\newblock An improved analysis of (variance-reduced) policy gradient and
  natural policy gradient methods.
\newblock \emph{Advances in Neural Information Processing Systems}, 33, 2020.

\bibitem[Lyu et~al.(2020)Lyu, Qi, Ghavamzadeh, Yao, Yang, and Liu]{mgh}
D.~Lyu, Q.~Qi, M.~Ghavamzadeh, H.~Yao, T.~Yang, and B.~Liu.
\newblock Variance-reduced off-policy memory-efficient policy search, 2020.

\bibitem[Marbach and Tsitsiklis(2001)]{marbach2001simulation}
P.~Marbach and J.~N. Tsitsiklis.
\newblock Simulation-based optimization of markov reward processes.
\newblock \emph{IEEE Transactions on Automatic Control}, 46\penalty0
  (2):\penalty0 191--209, 2001.

\bibitem[Mei et~al.(2020)Mei, Xiao, Szepesvari, and Schuurmans]{mei2020}
J.~Mei, C.~Xiao, C.~Szepesvari, and D.~Schuurmans.
\newblock On the global convergence rates of softmax policy gradient methods.
\newblock In \emph{International Conference on Machine Learning}, volume 119 of
  \emph{Proceedings of Machine Learning Research}, pages 6820--6829. PMLR,
  13--18 Jul 2020.

\bibitem[Mohammadi et~al.(2021)Mohammadi, Soltanolkotabi, and
  Jovanović]{mohammadi2021}
H.~Mohammadi, M.~Soltanolkotabi, and M.~R. Jovanović.
\newblock On the linear convergence of random search for discrete-time lqr.
\newblock \emph{IEEE Control Systems Letters}, 5\penalty0 (3):\penalty0
  989--994, 2021.

\bibitem[Nesterov and Spokoiny(2017)]{nesterov2017}
Y.~Nesterov and V.~Spokoiny.
\newblock Random gradient-free minimization of convex functions.
\newblock \emph{Foundations of Computational Mathematics}, 17:\penalty0 527--
  566, 2017.
\newblock ISSN 1615-3383.

\bibitem[Nesterov(2004)]{nesterov_book}
Y.~E. Nesterov.
\newblock \emph{Introductory Lectures on Convex Optimization - {A} Basic
  Course}, volume~87 of \emph{Applied Optimization}.
\newblock Springer, 2004.
\newblock ISBN 978-1-4613-4691-3.

\bibitem[Papini et~al.(2018)Papini, Binaghi, Canonaco, Pirotta, and
  Restelli]{papini2018}
M.~Papini, D.~Binaghi, G.~Canonaco, M.~Pirotta, and M.~Restelli.
\newblock Stochastic variance-reduced policy gradient.
\newblock In \emph{International Conference on Machine Learning}, volume~80 of
  \emph{Proceedings of Machine Learning Research}, pages 4026--4035. PMLR,
  10--15 Jul 2018.

\bibitem[Reddi et~al.(2016)Reddi, Hefny, Sra, P\'{o}cz\'{o}s, and
  Smola]{reddi16}
S.~J. Reddi, A.~Hefny, S.~Sra, B.~P\'{o}cz\'{o}s, and A.~Smola.
\newblock Stochastic variance reduction for nonconvex optimization.
\newblock In \emph{International Conference on Machine Learning}, pages
  314--323. JMLR.org, 2016.

\bibitem[Rubinstein(1969)]{rubinstein1969some}
R.~Y. Rubinstein.
\newblock \emph{Some problems in monte carlo optimization}.
\newblock PhD thesis, 1969.

\bibitem[Shamir(2017)]{shamir}
O.~Shamir.
\newblock An optimal algorithm for bandit and zero-order convex optimization
  with two-point feedback.
\newblock \emph{J. Mach. Learn. Res.}, 18\penalty0 (1):\penalty0 1703--1713,
  2017.
\newblock ISSN 1532-4435.

\bibitem[Shen et~al.(2019)Shen, Ribeiro, Hassani, Qian, and
  Mi]{shen2019hessian}
Z.~Shen, A.~Ribeiro, H.~Hassani, H.~Qian, and C.~Mi.
\newblock Hessian aided policy gradient.
\newblock In \emph{International Conference on Machine Learning}, pages
  5729--5738. PMLR, 2019.

\bibitem[Spall(1992)]{spall1992}
J.~C. Spall.
\newblock Multivariate stochastic approximation using a simultaneous
  perturbation gradient approximation.
\newblock \emph{IEEE Transactions on Automatic Control}, 37\penalty0
  (3):\penalty0 332--341, 1992.

\bibitem[Sutton and Barto(2018)]{sutton_book}
R.~S. Sutton and A.~G. Barto.
\newblock \emph{Reinforcement Learning: An Introduction}.
\newblock The MIT Press, 2 edition, 2018.

\bibitem[Sutton et~al.(1999)Sutton, McAllester, Singh, and
  Mansour]{sutton1999policy}
R.~S. Sutton, D.~A. McAllester, S.~P. Singh, and Y.~Mansour.
\newblock Policy gradient methods for reinforcement learning with function
  approximation.
\newblock In \emph{Advances in Neural Information Processing Systems},
  volume~99, pages 1057--1063, 1999.

\bibitem[Vijayan and Prashanth(2021)]{nv2021}
N.~Vijayan and L.~A. Prashanth.
\newblock Smoothed functional-based gradient algorithms for off-policy
  reinforcement learning: A non-asymptotic viewpoint.
\newblock \emph{Systems \& Control Letters}, 155:\penalty0 104988, 2021.
\newblock ISSN 0167-6911.

\bibitem[Williams(1992{\natexlab{a}})]{Williams1992}
R.~J. Williams.
\newblock Simple statistical gradient-following algorithms for connectionist
  reinforcement learning.
\newblock \emph{Machine Learning}, 8:\penalty0 229--256, 1992{\natexlab{a}}.

\bibitem[Williams(1992{\natexlab{b}})]{williams}
R.~J. Williams.
\newblock Simple statistical gradient-following algorithms for connectionist
  reinforcement learning.
\newblock \emph{Mach. Learn.}, 8:\penalty0 229--256, 1992{\natexlab{b}}.
\newblock ISSN 0885-6125.

\bibitem[Xu et~al.(2020)Xu, Gao, and Gu]{xu2020}
P.~Xu, F.~Gao, and Q.~Gu.
\newblock An improved convergence analysis of stochastic variance-reduced
  policy gradient.
\newblock In \emph{Uncertainty in Artificial Intelligence Conference}, volume
  115 of \emph{Proceedings of Machine Learning Research}, pages 541--551, 2020.

\bibitem[Zhang et~al.(2020{\natexlab{a}})Zhang, Kim, O'Donoghue, and
  Boyd]{zhangJ2020}
J.~Zhang, J.~Kim, B.~O'Donoghue, and S.~Boyd.
\newblock Sample efficient reinforcement learning with reinforce,
  2020{\natexlab{a}}.

\bibitem[Zhang et~al.(2020{\natexlab{b}})Zhang, Koppel, Zhu, and
  Basar]{zhangK2020}
K.~Zhang, A.~Koppel, H.~Zhu, and T.~Basar.
\newblock Global convergence of policy gradient methods to (almost) locally
  optimal policies.
\newblock \emph{{SIAM} J. Control. Optim.}, 58\penalty0 (6):\penalty0
  3586--3612, 2020{\natexlab{b}}.

\bibitem[Zhang et~al.(2019)Zhang, Boehmer, and Whiteson]{zhangS2019}
S.~Zhang, W.~Boehmer, and S.~Whiteson.
\newblock Generalized off-policy actor-critic.
\newblock In \emph{Advances in Neural Information Processing Systems},
  volume~32, pages 2001--2011, 2019.

\bibitem[Zhang et~al.(2020{\natexlab{c}})Zhang, Liu, Yao, and
  Whiteson]{zhangS2020}
S.~Zhang, B.~Liu, H.~Yao, and S.~Whiteson.
\newblock Provably convergent two-timescale off-policy actor-critic with
  function approximation.
\newblock In \emph{International Conference on Machine Learning}, volume 119 of
  \emph{Proceedings of Machine Learning Research}, pages 11204--11213. PMLR,
  13--18 Jul 2020{\natexlab{c}}.

\end{thebibliography}
\end{document}